\documentclass[10pt]{article} 
\usepackage[preprint]{tmlr}

\usepackage{amsmath,amsfonts,bm}

\def\eqref#1{equation~\ref{#1}}

\def\1{\bm{1}}

\def\va{{\bm{a}}}
\def\vb{{\bm{b}}}

\def\ve{{\bm{e}}}

\def\vh{{\bm{h}}}

\def\vx{{\bm{x}}}
\def\vy{{\bm{y}}}
\def\vz{{\bm{z}}}

\def\mA{{\bm{A}}}
\def\mB{{\bm{B}}}

\def\mG{{\bm{G}}}

\def\mI{{\bm{I}}}

\def\mR{{\bm{R}}}

\def\mW{{\bm{W}}}
\def\mX{{\bm{X}}}

\DeclareMathAlphabet{\mathsfit}{\encodingdefault}{\sfdefault}{m}{sl}
\SetMathAlphabet{\mathsfit}{bold}{\encodingdefault}{\sfdefault}{bx}{n}

\def\sR{{\mathbb{R}}}

\def\emG{{G}}

\usepackage[justification=raggedright,singlelinecheck=false]{caption}
\usepackage{amsthm}

\newtheorem{theorem}{Theorem}
\newtheorem{lemma}{Lemma}
\newtheorem{remark}{Remark}

\usepackage[normalem]{ulem}
\usepackage{booktabs}  
\usepackage{multirow}  
\usepackage{pifont}    
\usepackage{booktabs}
\usepackage[utf8]{inputenc} 
\usepackage[T1]{fontenc}    
\usepackage{hyperref}       
\usepackage{url}            
\usepackage{booktabs}       
\usepackage{amsfonts}       
\usepackage{nicefrac}       
\usepackage{microtype}      
\usepackage{xcolor}         
\usepackage{xcolor} 
\usepackage{multirow}
\usepackage{algorithmic}
\usepackage{subcaption}
\usepackage{xcolor}
\usepackage{svg}

\usepackage[ruled,vlined]{algorithm2e}
\usepackage{hyperref}
\usepackage{graphicx} 
\usepackage{pifont}   
\usepackage{array}    
\usepackage{amsmath}

\newcommand{\tcb}{\textcolor{blue}}
\usepackage[capitalize,noabbrev]{cleveref}
\usepackage{derivative}
\usepackage{gensymb}
\usepackage{enumitem}

\usepackage{hyperref}
\usepackage{url}
\newcommand{\pmstd}[2]{\ensuremath{#1\,{\scriptstyle\pm\,#2}}}

\title{Forward Target Propagation: A Forward-Only Approach to Global Error Credit Assignment via Local Losses}

\author{\name Nazmus Saadat As-Saquib \email saadatsaquib@tamu.edu \\
      \addr Texas A\&M University\\
      College Station, TX 77843
      \AND
      \name A N M Nafiz Abeer \email nafiz.abeer@tamu.edu \\
      \addr Texas A\&M University\\
      College Station, TX 77843
      \AND
      \name Hung-Ta Chien \email hchien000@tamu.edu\\
      \addr Texas A\&M University\\
      College Station, TX 77843
      \AND
      \name Byung-Jun Yoon \email bjyoon@tamu.edu\\
      \addr Texas A\&M University\\
      College Station, TX 77843
      \AND
      \name Suhas Kumar \email su1@alumni.stanford.edu\\
      \addr Nanyang Technological University, Singapore\\
      \addr Rain AI, San Francisco, CA
      \AND
      \name Su-in Yi \email yisuin@tamu.edu\\
      \addr Texas A\&M University\\
      College Station, TX 77843}

\def\month{10}  
\def\year{2026} 
\def\openreview{\url{https://openreview.net/forum?id=BN9uZVb7wQ}} 

\begin{document}

\maketitle
\let\AND\relax   

\begin{abstract}
Training neural networks has traditionally relied on backpropagation (BP), a gradient-based algorithm that, despite its widespread success, suffers from key limitations in both biological and hardware perspectives. These include backward error propagation by symmetric weights, non-local credit assignment, update locking, and frozen activity during backward passes. We propose Forward Target Propagation (FTP), a biologically plausible and computationally efficient alternative that replaces the backward pass with a second forward pass. FTP estimates layer-wise targets using only feedforward computations, eliminating the need for symmetric feedback weights or learnable inverse functions, hence enabling modular and local learning. We evaluate FTP on fully connected networks, CNNs, and RNNs, demonstrating accuracies competitive with BP on fully connected and sequential tasks, and consistently outperforming existing forward-only and target-propagation baselines across all architectures, including convolutional networks on CIFAR-10 and CIFAR-100. In simulated low-precision and noisy-weight settings that emulate analog in-memory hardware, FTP shows improved robustness relative to BP, and Multiply-and-Accumulate (MAC)-count analysis indicates substantially lower computational cost than other biologically inspired methods such as target propagation variants and forward-only learning algorithms. With its minimal computational overhead, forward-only nature, and hardware compatibility, FTP provides a promising direction for energy-efficient on-device learning and neuromorphic computing.
\end{abstract}

\section{Introduction}

Backpropagation (BP) has been the foundational algorithm for training neural networks, driving the success of deep learning in tasks such as image recognition, language modeling, and decision-making \citep{rumelhart1986backprop, attention_vashwani, feweshot_language}. Despite its proven effectiveness, BP faces key limitations in biological plausibility and hardware compatibility. A major issue is its reliance on symmetric weight transport between forward and backward passes, which conflicts with biological learning, where synaptic updates are not symmetric \citep{whittington2019credit}. Moreover, on edge devices and emerging analog in-memory hardware such as resistive random access memory (RRAM) and phase-change memory (PCM) crossbars, symmetric transport poses further challenges \citep{Yi2023, Yi_SONOS, ISAAC}. Each training iteration requires weight matrices to be written and verified in transposed form for backward error propagation. However, achieving perfect forward–backward symmetry in analog hardware is often unfeasible \citep{insitu_memristtor}, leading to performance degradation. This issue is amplified in edge applications, where low-bit precision constraints worsen the problem \citep{lowbit_issues}. Additionally, BP propagates a global error signal backward through multiple layers, unlike the localized learning signals of biological neural systems \citep{Lillicrap2020}. Its weight updates depend on information from distant layers, implying neurons must access far-off signals, a non-locality absent in biology. BP also suffers from vanishing and exploding gradients, where values become too small for meaningful updates or too large, destabilizing training, especially in deep and recurrent networks \citep{bengio1994learning}. Furthermore, BP assumes layer-wise processing and exact error computation, whereas biological systems learn in a noisier, more heuristic manner rather than through precise gradients \citep{crick1989recent, scellier2023energy_suhas}. These challenges have driven interest in biologically plausible alternatives that better capture how the brain may perform credit assignment \citep{Lillicrap2020, lillicrap2016random, bengio2014target, Hinton2022TheFA, PredictiveCA}. This paradigm is worth exploring because the brain achieves remarkable efficiency and adaptability through local learning rules, providing valuable insights for developing energy-efficient learning algorithms.

In this work, we introduce Forward Target Propagation (FTP), a biologically inspired, forward-only learning framework that addresses key limitations of backpropagation and demonstrate its theoretical basis, empirical competitiveness, efficiency across architectures and tasks, scalability among forward learning approaches, and compatibility with emerging hardware. 
Our key contributions are summarized as:
\begin{itemize}[topsep=0pt,itemsep=1pt,parsep=2pt,partopsep=0pt,leftmargin=1em]
    \item We introduce Forward Target Propagation (FTP), a biologically plausible, forward-only learning algorithm that enables local credit assignment without relying on backward error propagation or symmetric weight transport.
    \item We evaluate FTP across fully connected, convolutional, and recurrent neural networks on image classification and multivariate time-series forecasting tasks, showing that FTP matches or exceeds the recent bioplausible baselines at lower computational cost, approaches BP on fully connected and recurrent tasks, and narrows the gap to BP on convolutional benchmarks.
    \item We present theoretical justification and empirical evidence showing a steady alignment between FTP and backpropagation, measured via the angular similarity between their respective gradient directions during training.
    \item  We assess FTP's robustness under simulated low-precision and noisy-weight conditions, and estimate its computational cost on TinyML benchmarks, showing that FTP attains near-BP-level Multiply-and-Accumulate (MAC) counts while requiring substantially fewer operations than other biologically plausible algorithms. These results suggest its suitability for edge and neuromorphic computing.
\end{itemize}

\section{Related Work}


\subsection{Conventional and Biologically Plausible Methods}


Backpropagation (BP) has long been the backbone of neural network training, relying on forward and backward passes to compute error gradients. In the forward pass, inputs propagate through the network to the output layer, where prediction errors are computed against target values. The backward pass then sends error derivatives through the network using the same weights used in the forward pass, and the outer product of activity vectors from both passes forms the gradient matrices of the global objective. Despite its success, BP is often criticized for its biological implausibility, raising questions about whether such mechanisms exist in real biological systems \citep{rumelhart1986backprop}. To address the limitations of backpropagation, several biologically inspired learning algorithms have been proposed. Feedback Alignment (FA) \citep{lillicrap2016random} replaces symmetric feedback with fixed random weights, demonstrating that networks can learn without exact weight mirroring. Direct Feedback Alignment (DFA) extends this idea by projecting output errors directly to hidden units via random matrices \citep{nokland2016direct}. A simplified variant, Direct Random Target Projection (DRTP), further removes the need for global error signals by using target labels to update hidden layers \citep{frenkel2021learning}. Although these methods mitigate issues such as weight transport and update locking, they remain limited in biological plausibility due to their reliance on global error signals and frozen network dynamics. Moreover, DRTP has shown reduced accuracy and poor scalability compared to more recent approaches.

Equilibrium Propagation offers a biologically inspired alternative to backpropagation by computing gradients through equilibrium states in energy-based models \citep{scellier2017equilibrium, scellier2023energy_suhas, Yi2023}. Target Propagation (TP) and Difference Target Propagation (DTP) assign credit via layer-wise targets rather than global error gradients, thus avoiding a full backward pass \citep{bengio2014target, lee2015difference, dtp_meulemans, dtp_maxence}. While DTP improves target estimation with correction terms, both methods become costly in deep networks. Local Representation Alignment (LRA-E) refines this idea by computing local targets from downstream errors through fixed feedback pathways, removing the need for gradients or symmetric weights \citep{orobia2021local}. Its recursive variant, rec-LRA, scales this mechanism by aligning internal representations layer-by-layer, enabling parallel updates in deep architectures \citep{orobia2021reclra}. Though effective on large datasets, both methods introduce additional computational and memory overhead and need careful tuning for stability. Fixed Weight DTP (FWDTP) simplifies DTP with fixed random feedback matrices but suffers performance degradation in deeper networks \citep{shibuya2021fixed}.

These biologically inspired training schemes aim to match BP’s performance while mitigating the weight transport problem. However, challenges such as non-local credit assignment, frozen activity, update locking, and high computational cost remain, underscoring the need for further exploration in this domain.

\subsection{Forward Learning Methods}
Forward learning methods offer a biologically plausible alternative to backpropagation by avoiding explicit backward passes and instead utilizing feedforward mechanisms to propagate learning signals. 
PEPITA is one such forward learning method \citep{pepita_della}. It avoids backpropagation by performing a second forward pass where inputs are modulated by output errors through a fixed random matrix \citep{pepita_della}. Weight updates come from the difference between standard and modulated activations, enabling local learning without symmetric weights. This removes the weight symmetry constraint of backpropagation but faces the issue of scalability beyond two hidden layers. Forward-Forward Learning (FF), introduced by Hinton \citep{Hinton2022TheFA}, eliminates backward error propagation by training solely with forward passes. It compares neuron activations under ``positive'' and ``negative'' data, reinforcing those aligned with the desired output. By avoiding backward gradients, FF addresses weight transport and frozen activity. Yet the absence of top-down influence limits inter-layer coordination, partly alleviated by adding recurrent connections at the expense of higher computational complexity. FF also struggles to scale and underperforms beyond simple fully connected (FC) networks, such as convolutional neural networks (CNNs) \citep{dendritic_ll}. Predictive Forward-Forward Learning (PFF) extends FF with a generative circuit that predicts neural activity to guide local updates \citep{Ororbia2023ThePF}. Although more biologically plausible and less reliant on error feedback than predictive coding, PFF incurs extra cost from repeated prediction steps.

Despite these advancements, forward learning methods remain limited in efficiency, performance, and scalability. Although they address key biological implausibilities of BP such as symmetric weight transport and global error gradients, they often introduce new issues, including greater computational demands and poor scaling in deeper networks.

\section{Forward Target Propagation}
\label{section_FTP_three}
\subsection{Overview of Algorithm}
Forward Target Propagation (FTP) is a novel algorithm that replaces the backward pass in neural networks with a second forward pass to compute learning signals. At the core of FTP lies the concept of layer-wise targets similar to target propagation algorithms \citep{bengio2014target, lee2015difference, dtp_meulemans, shibuya2021fixed}, which define the desired activations that each layer should produce during learning. The process begins with a standard forward pass (first forward pass), where the input data \( \mX \) is propagated through the network to generate the output activation \( \vh_L \) based on the current weights. After obtaining this prediction, the target for the first hidden layer, denoted as \( \boldsymbol{\tau}_1 \), is estimated using a difference-corrected projection. Specifically, the final-layer output \( \vh_L \) and the ground-truth target \( \vy \) are independently projected through a fixed random matrix \( \mG \), initialized with zero mean and small standard deviation. The resulting estimate is computed as:
\begin{equation}
\label{first_eqn}
\boldsymbol{\tau}_1 = \sigma(\mG \vy) - \sigma(\mG \vh_L) + \vh_1, 
\end{equation}
where \( \sigma(\cdot) \) is the element-wise nonlinearity, and \( \vh_1 \) is the activation of the first hidden layer from the first forward pass. The estimated targets then serve as inputs for the second forward pass, during which the same feedforward weights used in the first forward computation are reused to propagate the target signal and estimate layer-wise targets. Finally, the weights are updated by minimizing a local loss between activations and targets, thereby enabling layer-wise learning without backward gradients or symmetric feedback.

This forward-only update features a pipelined training by overlapping minibatches: once the first layer’s weights are updated and the next layer’s target is being computed, the following minibatch can begin its first forward pass. In contrast, BP requires all layers to complete backward gradient computation and weight updates before processing the next minibatch, resulting in strict update locking. FTP thus allows overlapping computation of minibatches, thereby partially solves update locking and can lead to improved wall-clock efficiency over a full training cycle than BP. For comparison, the configuration of various learning methods- BP, DTP, PEPITA, and FTP, is shown in Figure \ref{fig:architecture_comparison}.

\begin{figure*}[t!]
\vskip -.3in
\centering
\begin{subfigure}{0.212\textwidth }
    \centering
    \includegraphics[width=1\linewidth]{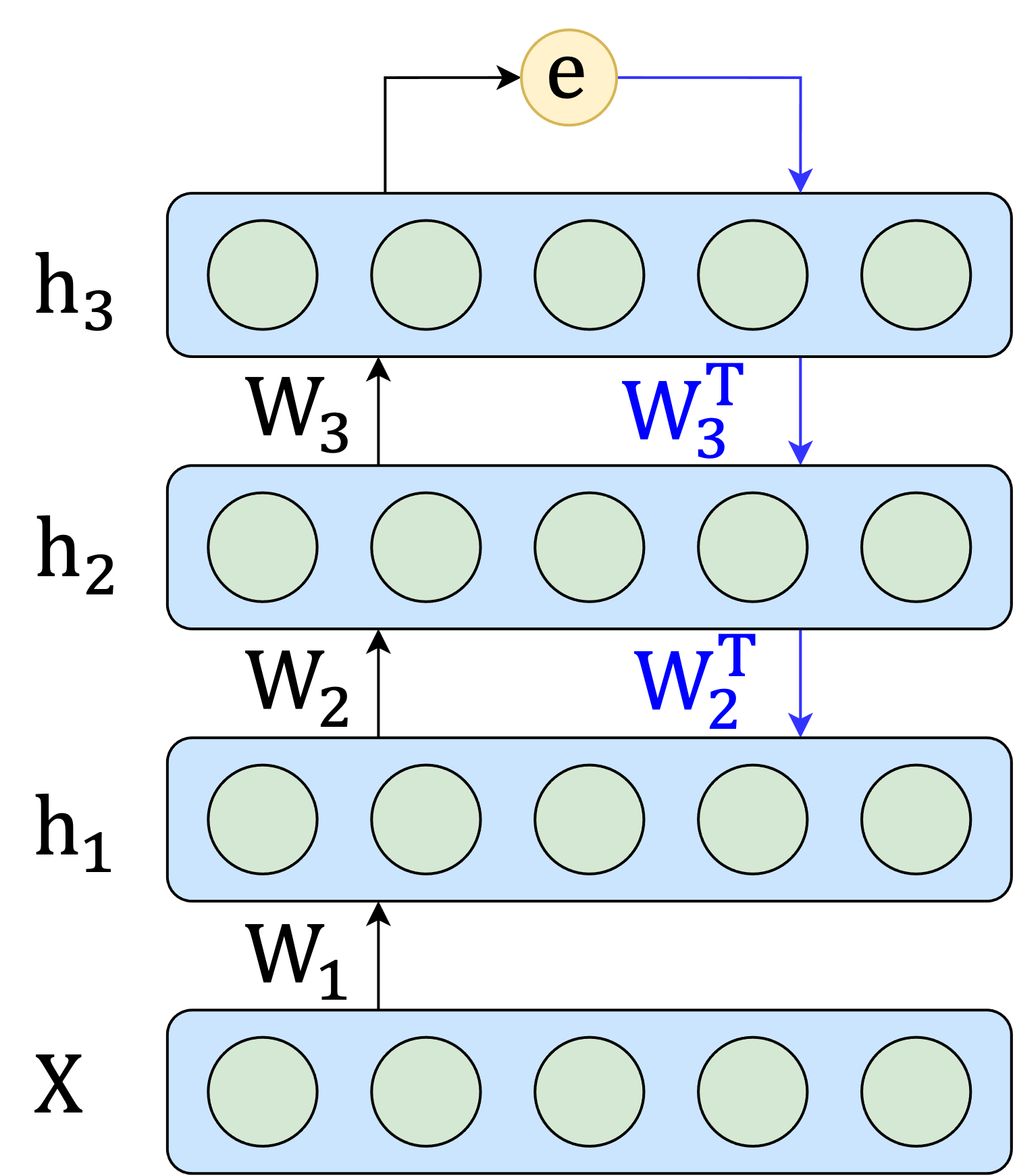}  
    \caption{BP}
    \label{fig:BP_arch}
\end{subfigure}
\hfill
\begin{subfigure}{0.201\textwidth}
    \centering
    \includegraphics[width=0.88\linewidth]{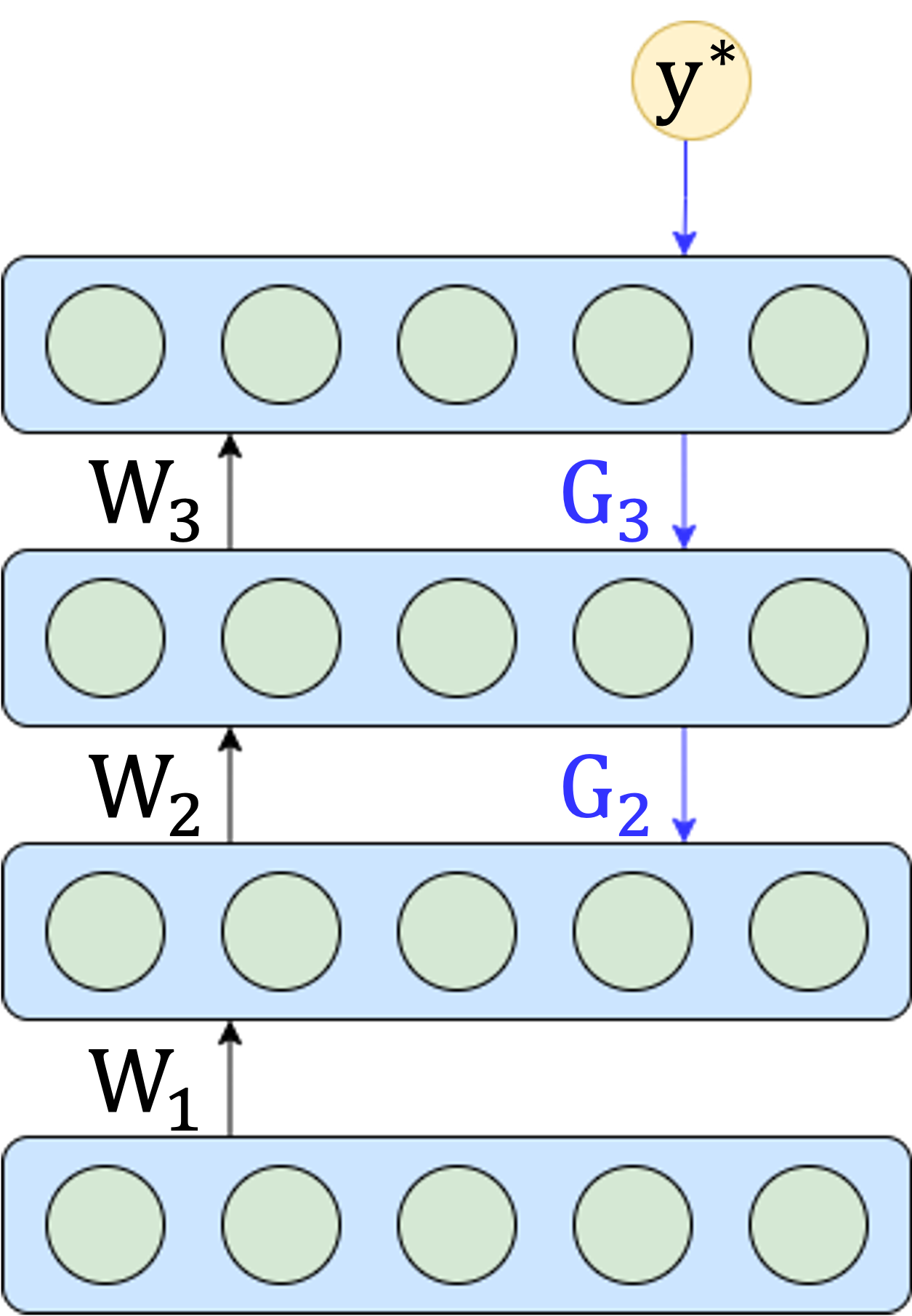}  
    \caption{DTP}
    \label{fig:dtp_arch}
\end{subfigure}
\hfill
\begin{subfigure}{0.209\textwidth}
    \centering
    \includegraphics[width=1.05\linewidth]{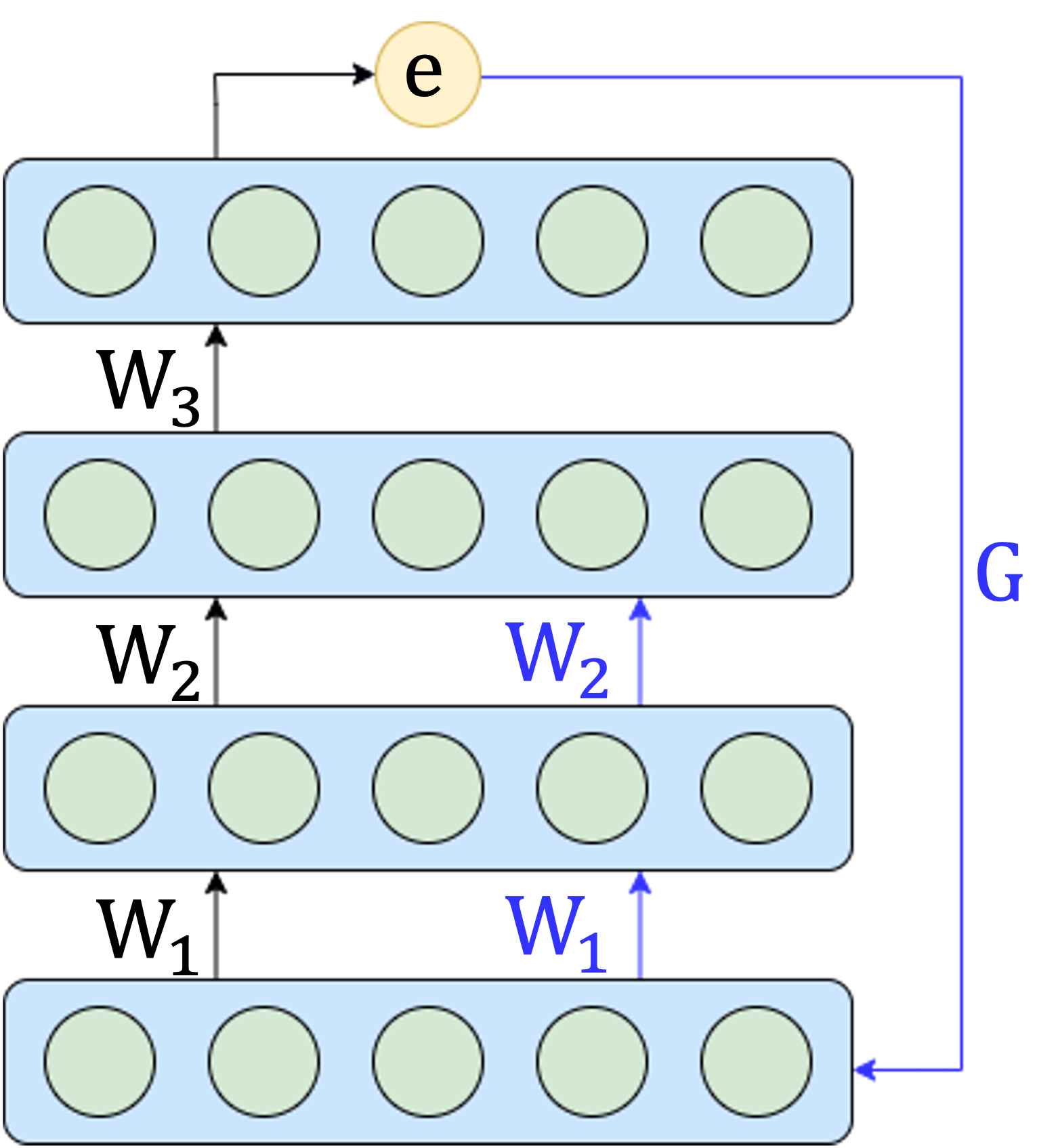}  
    \caption{PEPITA}
    \label{fig:pepita_arch}
\end{subfigure}
\hfill
\begin{subfigure}{0.207\textwidth}
    \centering
    \includegraphics[width=1.087\linewidth]{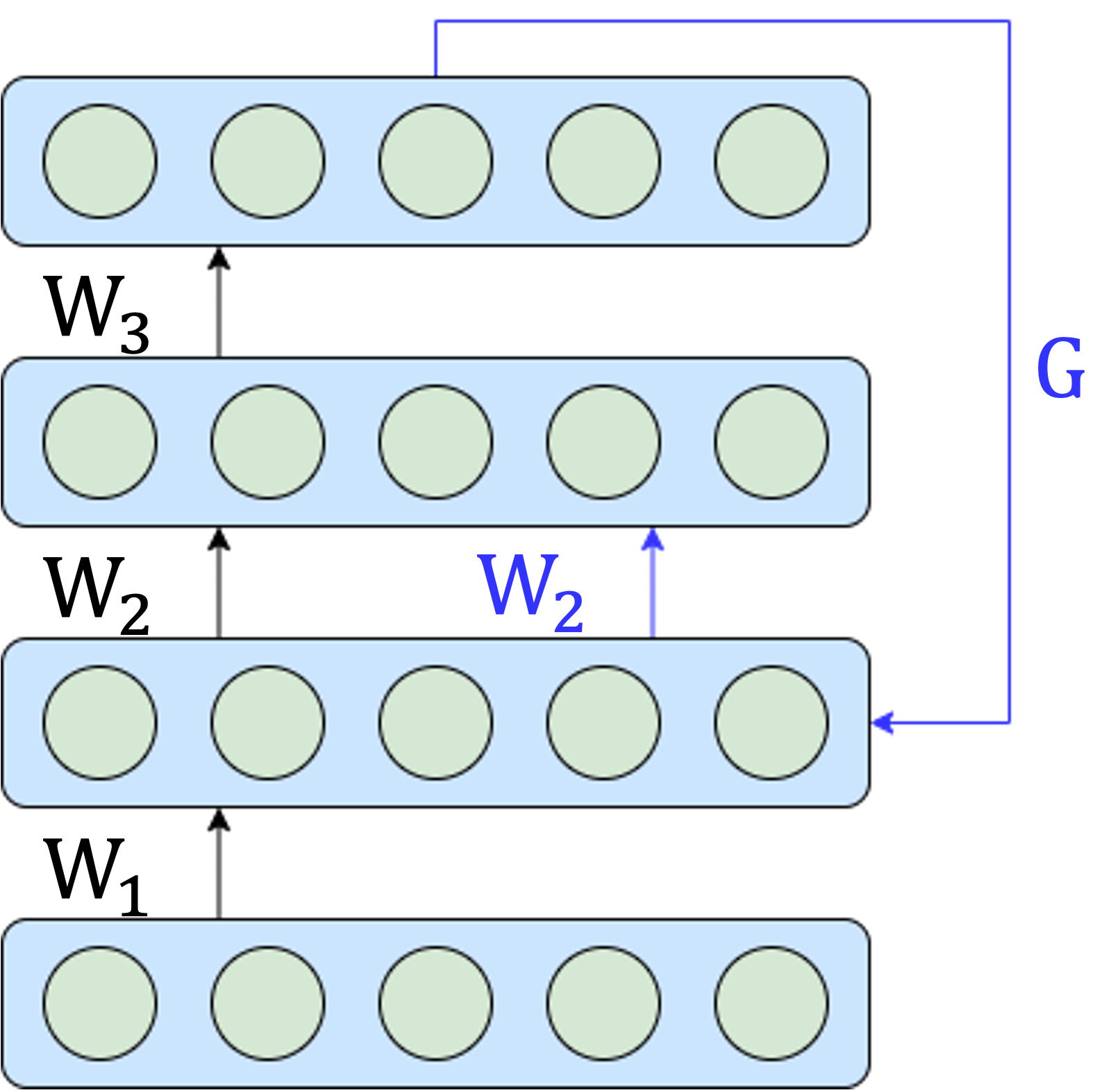}  
    \caption{FTP}
    \label{fig:ftp_arch}
\end{subfigure}
\vskip -3mm
\caption{Configuration of various learning methods: a) BP, b) DTP, c) PEPITA, and d) FTP. Black arrows represent forward paths, while blue arrows indicate error/target paths. Each $\mW$ denotes the learnable weight matrices of hidden layers, whereas the top-down feedback path includes \textbf{learnable} matrices ($\mG_3,\mG_2$ for DTP) or \textbf{randomly initialized and frozen} matrix $\mG$ (for PEPITA and FTP).}
\label{fig:architecture_comparison}
\vskip -.2in
\end{figure*}

\subsection{Target Estimation and Learning}
\label{learning_method}
In a feedforward neural network, the process begins with input data \( \mX \) and a corresponding output-layer target \( \vy \). In supervised classification \( \vy \) is the class label, but FTP places no such restriction on it; any target defined at the output layer is admissible. The activation values for the input layer (layer 0) are given by: \( \vh_0 = \mX \).
For each subsequent layer \( i \) (where \( i = 1, 2, \ldots, L \), and \( L \) is the final layer), the activation values are computed as:
\begin{equation}
\vh_i = \sigma(\mW_i \vh_{i-1}) 
\end{equation}
where \( \mW_i \) is the weight matrix for layer \( i \), \( \vh_{i-1} \) is the activation from the previous layer, and \( \sigma \) is the activation function applied element-wise (e.g., \textit{tanh}, sigmoid).\\
FTP assigns each layer a target via forward propagation of the initial target signal. We estimate the first target \( \boldsymbol{\tau}_1 \), which is the target for the first hidden layer, using a fixed random matrix \( \mG \), through a difference-corrected projection that contrasts the label and current output in a transformed space:
\begin{equation}
\boldsymbol{\tau}_1 = \sigma(\mG {\vy}) - \sigma(\mG \vh_L) + \vh_1  \label{eqn: first_target}
\end{equation}
This formulation provides a top-down feedback mechanism to estimate layer-wise targets without relying on the global error gradient propagation. The target for the first hidden layer from \cref{eqn: first_target} reflects how its activation should shift to ultimately produce the correct output, \( \vh_L \approx \vy \). Although \( \vh_1 \in \mathbb{R}^{d_1} \) and \( \vh_L, \vy \in \mathbb{\mR}^{d_L} \), the fixed random projection \( \mG \) maps \(\vh_L, \vy \) into \( \mathbb{\mR}^{d_1} \). The difference \( \sigma(\mG\vy) - \sigma(\mG\vh_L) \) captures the directional discrepancy between the desired output and current prediction in the projected space. Adding this correction to \( \vh_1 \) yields a target \( \boldsymbol{\tau}_1 \) that nudges the first hidden layer toward a representation more likely to result in \( \vh_L \approx \vy \) after forward propagation.

To transmit this signal to deeper layers, i.e., \( i = 2, 3, \ldots, L-1 \), we propagate targets recursively using the same feedforward weights used in the first forward pass:
\begin{equation}
\boldsymbol{\tau}_i = \sigma(\mathbf{W}_i \boldsymbol{\tau}_{i-1}), 
\end{equation}

Each layer's weights are then updated by minimizing a local loss that encourages alignment between activation and target:
\begin{equation}
\mathcal{L}_i = 
\begin{cases}
\left\| \vh_i - \operatorname{stop\_grad}(\boldsymbol{\tau}_i) \right\|_2^2 & \text{if } i < L \\
\mathcal{L} & \text{if } i = L, \label{eq:localloss}
\end{cases}
\end{equation}
where \( \mathcal{L} \) is the global loss (e.g., cross-entropy, mean-squared loss) used at the final layer. The operator \( \operatorname{stop\_grad}(\cdot) \) marks the target \( \boldsymbol{\tau}_i \) as a constant during differentiation: although both \( \vh_i \) and \( \boldsymbol{\tau}_i \) depend on \( \mW_i \), the gradient is taken only with respect to the activation. 
The FTP weight updates for hidden layers that are obtained by performing standard gradient descent on the local loss \( \mathcal{L}_i \), leads to:
\begin{equation}
\Delta \mW_i = -\eta \,\big[(\vh_i - \boldsymbol{\tau}_i)\odot\sigma'(\mW_i\vh_{i-1})\big]\,\vh_{i-1}^\top, \label{eqn: w_update}
\end{equation}
where \( \eta \) is the learning rate. 


The formulation of FTP is closely related to Indirect Feedback Alignment (IFA) \citep{nokland2016direct}, which projects the output error to the first hidden layer through a fixed random matrix and propagates it upward through the forward weights. Under linear activations, FTP updates for the weights equal the corresponding IFA update, so the linear-regime analysis of Appendix~\ref{appendix_theory} applies to IFA as well. With nonlinear activations, FTP propagates targets rather than errors: each layer's learning signal is the finite difference between two passes through the nonlinear forward map, instead of a linearized error. A detailed comparison of the two methods is given in Appendix~\ref{app:ifa}.

FTP enables weight updates using only locally available information, where each update implicitly reflects a top-down influence from the global target through fixed projections. FTP also eliminates the need to compute the first hidden layer’s activation from the input level in second forward pass, unlike other forward-learning algorithms that do so using modulated inputs or negative image samples \citep{pepita_della, Hinton2022TheFA}. Furthermore, as shown in Appendix \ref{appendix_theory} with more restricted conditions, this mechanism introduces the global error $\ve$ through this forward target propagation, which produces updates that align closely with the BP gradients and exhibit similarity to Gauss-Newton directions. 
The complete training procedure is outlined in Algorithm~\ref{alg:FTP}.

FTP also differs from existing forward-learning methods in its degree of biological plausibility. PEPITA, for example, computes an explicit global error at the output layer and projects it back to the input using a fixed feedback matrix. A second forward pass with perturbed input induces a modified hidden activation, which is compared to the original activations to drive learning. This requires explicit error computation and assumes sensory neurons can access and respond to internal errors- an implausible assumption, as early sensory neurons act as feedforward encoders without receiving task-level feedback ~\citep{Gilbert2013}.  Moreover, explicit error signaling is biologically unsupported, as noted by Guerguiev et al.~\citep{guerguiev2017towards}, who argue for dendritic compartmentalization as the mechanism for integrating feedforward and feedback signals. In contrast, FTP generates the target for the first hidden layer by combining its feedforward activation $\vh_1$ with a top-down projection of the output-layer target through $\mG$, avoiding explicit error computation and input perturbation. This is consistent with cortical pyramidal neuron structure, where basal dendrites receive feedforward inputs and apical dendrites receive top-down feedback, supporting local plasticity through dendritic mismatch~\citep{guerguiev2017towards}. This correspondence applies to the first hidden layer, which is the only site at which FTP introduces a feedback signal. For layers $2, \dots, L-1$, FTP requires no feedback pathway at all. The target $\boldsymbol{\tau}_i = \sigma(\mW_i \boldsymbol{\tau}_{i-1})$ is carried by the same feedforward synapses that compute $\vh_i$,  so

\vspace{-2mm}
\noindent
\begin{minipage}[t]{0.55\textwidth}
each neuron expresses two activity states at different times, a prediction state in the first forward pass and a target state in the second, with plasticity driven by their mismatch. \citet{kohan2018efp} identify precisely this as the core mechanism shared by biologically plausible learning rules: a hidden neuron expressing its current prediction firing rate and its target firing rate at two sufficiently separated timescales, such that closing the gap between them drives the output toward the target. They further argue, as does \citet{crick1989recent}, that the problematic assumption in backpropagation and its variants is not weight symmetry alone but the presupposition of systematic feedback connectivity to every neuron, for which the neocortex provides no evidence. FTP requires a single fixed top-down projection to one layer, whereas BP requires transposed weights at every layer, DTP a learned inverse at every layer, and FA a random feedback matrix at every layer. This is consistent with \citet{kohan2022SigProp}, who show that neurons lacking feedback connections can nonetheless receive a global learning signal through their feedforward connectivity.
\end{minipage}
\hfill
\begin{minipage}[t]{0.43\textwidth}
\vspace{-2mm} 

\begin{algorithm}[H]
   \caption{FTP}
   \label{alg:FTP}

\begin{algorithmic}
   \STATE {\bfseries Given:} Input $\mX$, output target $\vy$
   \STATE $\vh_0 = \mX$, $\boldsymbol{\tau}_L = \vy$
   \STATE \textcolor{gray}{\# First Forward Pass}
   \FOR{$i = 1$ to $L$}
       \STATE $\vh_i = \sigma(\mW_i \vh_{i-1})$
   \ENDFOR
   \STATE \textcolor{gray}{\# Estimate First Target}
\STATE $\boldsymbol{\tau}_1 = \sigma(\mG \boldsymbol{\tau}_L) - \sigma(\mG \vh_L) + \vh_1$
\STATE \textcolor{gray}{\# Second Forward Pass}
\FOR{$i = 1$ to $L$}
    \IF{$1 < i < L$}
        \STATE $\boldsymbol{\tau}_i = \sigma(\mW_i \boldsymbol{\tau}_{i-1})$
    \ENDIF
    \STATE $\Delta \mW_i = -\eta \,\big[(\vh_i - \boldsymbol{\tau}_i)\odot\sigma'(\mW_i\vh_{i-1})\big]\,\vh_{i-1}^\top$ \quad
\ENDFOR
\end{algorithmic}
\end{algorithm}
\end{minipage}

\vspace{-1.8mm}

\section{Results and Discussion}

\subsection{Experimental Setup}

We evaluated the performance of FTP against conventional BP and other biologically plausible algorithms such as PEPITA and DTP across two task categories: image classification using FC and CNN architectures, and time-series forecasting using recurrent neural networks (RNNs). FTP's performance on a task without labels, such as an autoencoder, is presented in Appendix~\ref{app:ae}, where the objective is reconstruction rather than classification.

All models were trained using cross-entropy loss for classification tasks and mean squared error for time-series forecasting, optimized by SGD with momentum. For consistency across evaluations, each model family (FC, CNN, and RNN) was implemented using a consistent architecture across all corresponding datasets. Architectural specifications, activation functions, initialization schemes, and training schedules are provided in the Appendix \ref{appendix_model_details}. Furthermore, we evaluate the impact of initialization choice of feedback matrix, $\mG$ which is elaborated in Appendix \ref{sensitivity_G_appendix}.

\subsection{FTP for Image Classification}
\label{sectionfc&cnn}
\label{original_results}

We evaluated FTP on FC and CNN networks for image classification. FC models were tested on MNIST, FMNIST, and CIFAR-10, while CNNs were evaluated on MNIST, CIFAR-10, and CIFAR-100. Results are summarized in Table~\ref{test_acc}. Across FC architectures, FTP consistently outperforms DTP and achieves accuracy comparable to BP and PEPITA, but with significantly lower computational cost as shown in Table~\ref{MAC_origin}. On CNNs, FTP matches or exceeds the bio-plausible baselines, performing on par with PEPITA and well ahead of DTP, which demonstrates its ability to learn spatial features effectively.

\begin{table*}[h]
\vspace{-3mm}
    \caption{Test accuracy [\%] achieved by BP, DTP, PEPITA, and FTP in the experiments for fully connected (FC) networks and convolutional neural networks (CNNs).}
    \label{test_acc}
    \centering
    \vspace{-3mm}
    \small
    \begin{tabular}{lccc|ccc}
        \toprule
        & \multicolumn{3}{c|}{\textbf{FC Networks}} & \multicolumn{3}{c}{\textbf{CNNs}} \\
        Algorithm & MNIST & FMNIST & CIFAR-10 & MNIST & CIFAR-10 & CIFAR-100 \\
        \midrule
        BP      & 98.27{\tiny$\pm$0.08} & 89.10{\tiny$\pm$0.12} & 55.31{\tiny$\pm$0.26}  & 98.74{\tiny$\pm$0.05} & 64.88{\tiny$\pm$0.18} & 33.83{\tiny$\pm$0.26} \\
        DTP     & 96.51{\tiny$\pm$0.34} & 85.87{\tiny$\pm$0.41} & 48.67{\tiny$\pm$0.19}  & 97.23{\tiny$\pm$0.22} & 52.76{\tiny$\pm$0.27} & 23.51{\tiny$\pm$0.58} \\
        PEPITA  & 98.05{\tiny$\pm$0.11} & 88.41{\tiny$\pm$0.14} & 52.45{\tiny$\pm$0.28}  & 98.41{\tiny$\pm$0.24} & 56.17{\tiny$\pm$0.62} & 26.77{\tiny$\pm$0.87} \\
        FTP     & 97.98{\tiny$\pm$0.25} & 87.24{\tiny$\pm$0.21} & 52.57{\tiny$\pm$0.37}  & 98.28{\tiny$\pm$0.35} & 56.32{\tiny$\pm$0.82} & 26.84{\tiny$\pm$1.13} \\
        \bottomrule
    \end{tabular}
    \vspace{-2mm}
\end{table*}


\begin{table*}[h]
\caption{MACs (millions) for BP, DTP, PEPITA, and FTP on FC networks, along with percentage change in MAC (\%) with respect to BP.}
\label{MAC_origin}
\vspace{-2mm}
\centering
\small
\begin{tabular}{llcccc}
\toprule
Dataset & \textbf{} & BP & DTP & PEPITA & FTP \\
\midrule
\multirow{2}{*}{MNIST}
    & MAC (mil.) & 2.00 & 2.94 & 2.81 & \textbf{2.02} \\
    & MAC (\%)   & 0.00\% & 47.00\% & 40.50\% & \textbf{1.00\%} \\
\cmidrule(lr){1-6}
\multirow{2}{*}{CIFAR-10}
    & MAC (mil.) & 6.69 & 9.97 & 9.86 & \textbf{6.71} \\
    & MAC (\%)   & 0.00\% & 49.03\% & 47.38\% & \textbf{0.30\%} \\
\cmidrule(lr){1-6}
\multirow{2}{*}{CIFAR-100}
    & MAC (mil.) & 6.72 & 10.01 & 10.18 & \textbf{6.93} \\
    & MAC (\%)   & 0.00\% & 48.96\% & 51.49\% & \textbf{3.12\%} \\
\bottomrule
\end{tabular}
\vskip -10pt
\end{table*}
Table~\ref{MAC_origin} also shows that FTP requires far fewer Multiply-and-Accumulate (MAC) operations per input sample than other bio-plausible methods such as PEPITA and DTP, while remaining close to BP in computational cost. This highlights FTP’s practical efficiency in shallow-to-moderately deep architectures, making it a compelling forward-only alternative for both dense and convolutional models, especially in resource-constrained settings.

\subsection{FTP for Sequential Tasks}
\label{section_rnn}
We evaluated FTP on three standard time-series forecasting benchmarks, Electricity, METR-LA (Traffic), and Solar Energy \citep{Lai2017ModelingLA}, using an RNN trained to predict the 25\textsuperscript{th} time sample from the previous 24. As in FC and CNN settings, FTP estimates targets using a difference-corrected projection through a fixed random matrix. Detailed formulation of RNNs with FTP and its update rule are presented in Appendix~\ref{app:ftp_rnn}. We evaluated model performance using two standard time-series metrics: Root Relative Squared Error (RRSE) and the Pearson correlation coefficient (CORR). These metrics are defined in Appendix \ref{appendix_rrse_corrr}.

\begin{table}[h]
\vspace{-1mm}
\centering
\caption{Performance of RNN-based models on time-series datasets using BP, PEPITA, and FTP. $\uparrow$ ($\downarrow$) indicates that higher (lower) values are better.}
\label{tab:rnn_results}
\small
\begin{tabular}{lcccccc}
\toprule
\multirow{2}{*}{Method} & \multicolumn{2}{c}{Electricity} & \multicolumn{2}{c}{METR-LA} & \multicolumn{2}{c}{Solar} \\
\cmidrule(lr){2-3} \cmidrule(lr){4-5} \cmidrule(lr){6-7}
 & RRSE$\downarrow$ & CORR$\uparrow$ & RRSE$\downarrow$ & CORR$\uparrow$ & RRSE$\downarrow$ & CORR$\uparrow$ \\
\midrule
BP (BPTT)         & \pmstd{0.1059}{0.0013} & \pmstd{0.9302}{0.0011} & \pmstd{0.4680}{0.0083} & \pmstd{0.8750}{0.0014} & \pmstd{0.1161}{0.0009} & \pmstd{0.9919}{0.0009} \\
PEPITA-RNN        & \pmstd{0.1214}{0.0019} & \pmstd{0.9929}{0.0009} & \pmstd{0.4419}{0.0065} & \pmstd{0.8967}{0.0022} & \pmstd{0.1170}{0.0017} & \pmstd{0.9932}{0.0022} \\
\textbf{FTP-RNN}  & \pmstd{0.1219}{0.0010} & \pmstd{0.9935}{0.0010} & \pmstd{0.4398}{0.0073} & \pmstd{0.8987}{0.0019} & \pmstd{0.1177}{0.0011} & \pmstd{0.9931}{0.0019} \\
\bottomrule
\end{tabular}
\vspace{-2mm}
\end{table}


RNNs are inherently deep due to their temporal unrolling, and Table~\ref{tab:rnn_results} shows that FTP consistently performs well across all datasets in RNNs. FTP has higher correlation coefficient than BP on all three benchmarks and remains competitive in RRSE. On METR-LA, FTP achieves both the lowest RRSE and highest CORR. These results suggest that FTP is particularly effective at capturing temporal structure and underlying patterns in time-series data, a crucial aspect of real-world forecasting tasks. We further evaluate longer input windows ($\{48, 96, 120\}$ timesteps) in Appendix~\ref{appendix_horizons}; FTP remains stable across all window lengths and continues to match or exceed BP in correlation. This highlights FTP's potential as a biologically plausible alternative to BPTT, with improved robustness and compatibility with analog or on-device learning environments, which is discussed in Section \ref{sec:robust_FTP}.

\subsection{Depth Scaling of FTP}
\label{section_scalability}
To assess the depth scaling of FTP, we evaluated both deep FC networks (all hidden layers of 1024 units) and convolutional architectures. Details on architecture choice are described in Appendix \ref{scalability_appendix}. As evident from Table \ref{tab:bp-pep-ftp}, across datasets, BP achieves the highest accuracy, but FTP consistently remains close to BP, while being far more stable and better than PEPITA, which collapses with depth. Prior work has shown that forward-only learning schemes without extra learnable parameters, such as FF, were restricted to at most four hidden layers \citep{Hinton2022TheFA}, while PEPITA was originally limited to one hidden layer \citep{pepita_della} and required auxiliary techniques such as normalization or weight decay to stabilize deeper FC models \citep{pepita_srini}. Potential scalability for CNNs using FTP is discussed in Appendix \ref{scalability_appendix}. Our results demonstrate that FTP can train both deeper FC and convolutional networks without additional learning \citep{monoforward} or stabilization techniques, and produces predictions directly in a single forward pass during inference, unlike FF, which requires testing all labels explicitly for each sample.

\begin{table}[h]
\vspace{-2mm}
\centering
\caption{Test accuracy (Mean {\scriptsize$\pm$ Std}) of BP, PEPITA, and FTP across datasets and hidden layer depths for FC networks. Here, L = Total number of layers in network.}
\label{tab:bp-pep-ftp}
\small
\renewcommand{\arraystretch}{1.2}
\begin{tabular}{l l c c c c}
\toprule
Dataset & Algorithm & L = 3 & L = 4 & L = 5 & L = 6 \\
\midrule
\multirow{3}{*}{MNIST} 
 & BP      & 98.48{\scriptsize$\pm$0.09} & 98.55{\scriptsize$\pm$0.05} & 98.60{\scriptsize$\pm$0.06} & 98.60{\scriptsize$\pm$0.05} \\
 & PEPITA  & 98.19{\scriptsize$\pm$0.03} & 94.63{\scriptsize$\pm$0.21} & 89.12{\scriptsize$\pm$0.35} & 81.71{\scriptsize$\pm$0.80} \\
 & FTP  & 98.36{\scriptsize$\pm$0.07} & 98.15{\scriptsize$\pm$0.11} & 98.24{\scriptsize$\pm$0.08} & 98.20{\scriptsize$\pm$0.07} \\
\midrule
\multirow{3}{*}{CIFAR-10} 
 & BP      & 57.83{\scriptsize$\pm$0.08} & 58.06{\scriptsize$\pm$0.19} & 57.86{\scriptsize$\pm$0.39} & 57.65{\scriptsize$\pm$0.37} \\
 & PEPITA  & 52.46{\scriptsize$\pm$0.17} & 52.78{\scriptsize$\pm$0.56} & 44.88{\scriptsize$\pm$0.19} & 32.72{\scriptsize$\pm$0.36} \\
 & FTP  & 53.88{\scriptsize$\pm$0.24} & 54.81{\scriptsize$\pm$0.18} & 54.72{\scriptsize$\pm$0.16} & 54.93{\scriptsize$\pm$0.13} \\
\midrule
\multirow{3}{*}{CIFAR-100} 
 & BP      & 30.08{\scriptsize$\pm$0.17} & 29.31{\scriptsize$\pm$0.28} & 28.50{\scriptsize$\pm$0.11} & 28.55{\scriptsize$\pm$0.29} \\
 & PEPITA  & 24.80{\scriptsize$\pm$0.32} & 18.05{\scriptsize$\pm$0.16} &  1.20{\scriptsize$\pm$0.11} &  6.57{\scriptsize$\pm$0.48} \\
 & FTP  & 26.27{\scriptsize$\pm$0.27} & 27.12{\scriptsize$\pm$0.20} & 27.85{\scriptsize$\pm$0.36} & 28.35{\scriptsize$\pm$0.54} \\
\bottomrule
\end{tabular}
\vspace{-1mm}
\end{table}

\subsection{Alignment of FTP with Backpropagation} 
\label{section_alignment}
Gradient alignment measures how closely alternative learning rules match BP's parameter update directions. This metric is particularly useful in the context of biologically plausible algorithms ~\citep{lillicrap2016random, nokland2016direct, metagradients}, where exact gradient computation is often replaced by approximate or learned feedback pathways. We quantify the alignment by computing the cosine angle between flattened gradient vectors of an approximate method and BP throughout training. A smaller angle  indicates a closer match in update directions, which often correlates with improved learning performance and better convergence. To analyze the gradient alignment of FTP with BP, we use a fully connected network with two hidden layers, trained on MNIST for 100 epochs, and average the results over 10 random seeds. A scaling parameter \(\gamma\), introduced in detail below, is set to 1 at this point, which corresponds to the standard FTP target. The alignment angle for the output layer always remains at zero during training, since FTP and BP share identical gradient expressions for the final layer.  As shown in Figure~\ref{fig:alignment_w1w2}, the alignment angle for both hidden layers starts near \(90^\circ\), indicating no initial alignment between the FTP and BP gradients.  This is expected, as the fixed random feedback matrix \(\mG \) used in FTP is uncorrelated with the forward weights at initialization, unlike BP, which computes exact gradients via the chain rule. As training progresses, however, the angle steadily decreases, which reflects that FTP's updates increasingly align with those of BP.

\begin{figure*}[t]
\vspace{-10pt}
    \centering
    \begin{subfigure}[t]{0.435\textwidth}
        \centering
        \includegraphics[width=\linewidth]{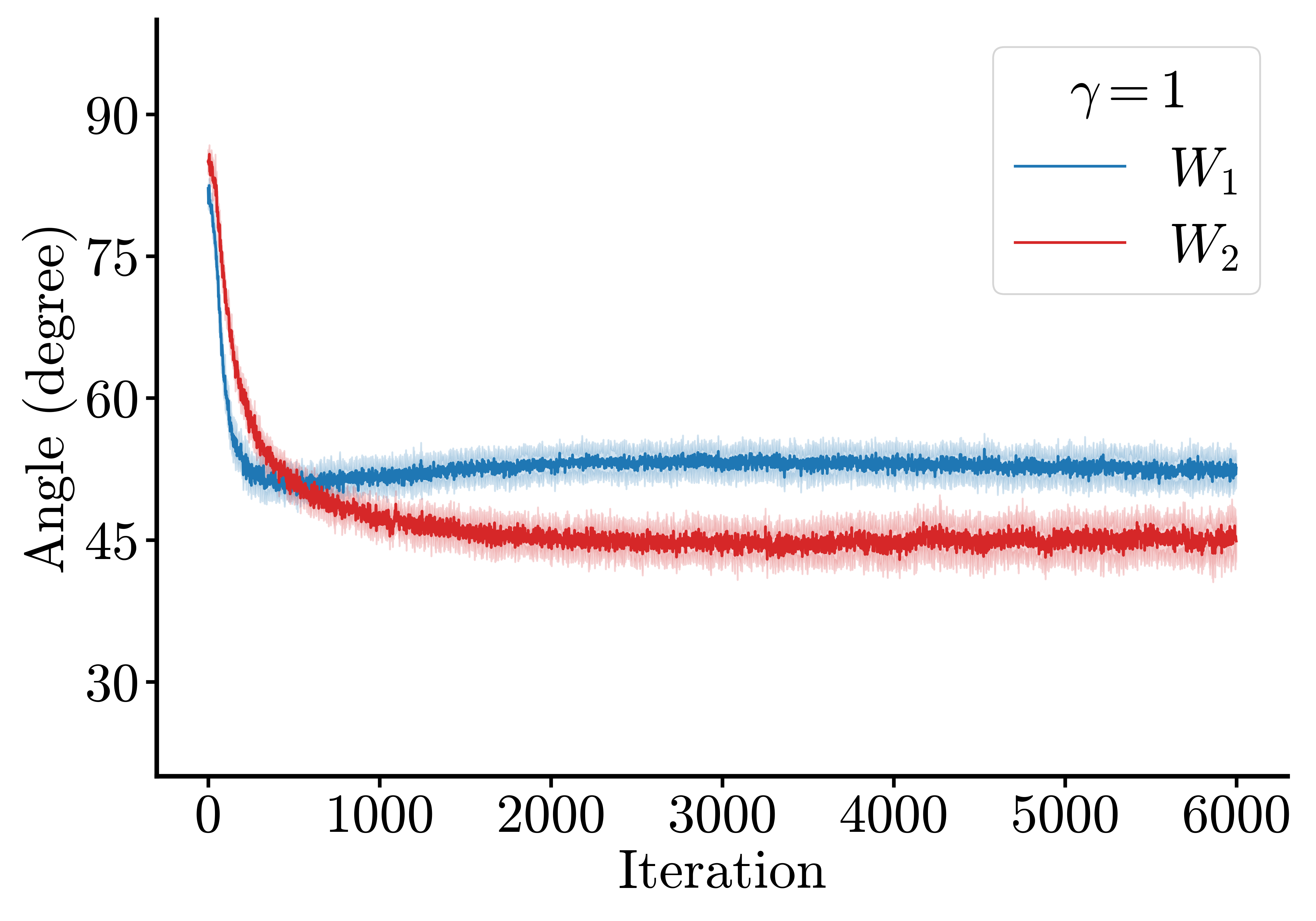}
        \caption{}
        \label{fig:alignment_w1w2}
    \end{subfigure}
    \hfill
    \begin{subfigure}[t]{0.435\textwidth}
        \centering
        \includegraphics[width=\linewidth]{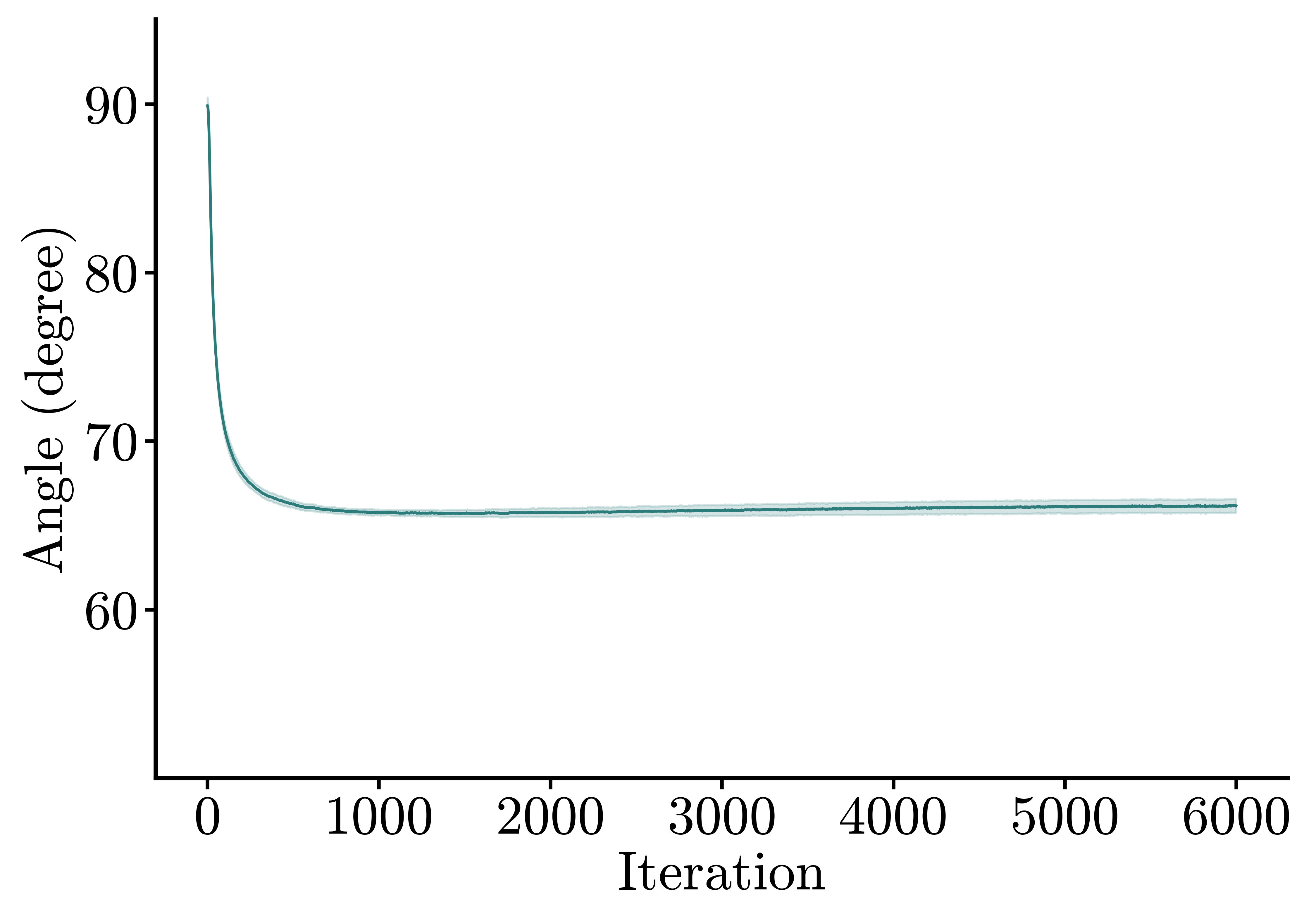}
        \caption{}
        \label{fig:alignment_wg}
    \end{subfigure}

    \vspace{-.5mm}

    \begin{subfigure}[t]{0.435\textwidth}
        \centering
        \includegraphics[width=\linewidth]{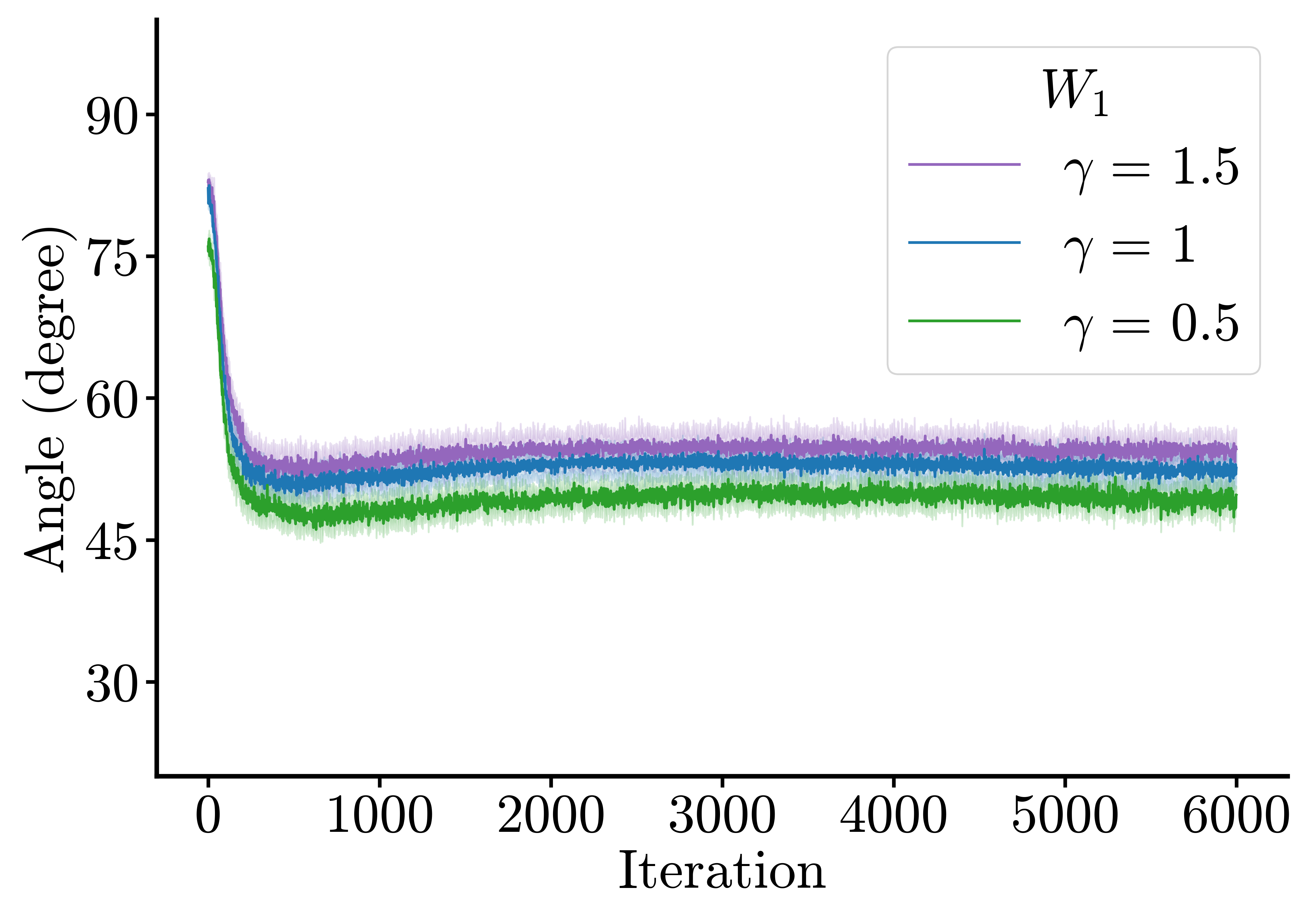}
        \caption{}
        \label{fig:alignment_w1}
    \end{subfigure}
    \hfill
    \begin{subfigure}[t]{0.435\textwidth}
        \centering
        \includegraphics[width=\linewidth]{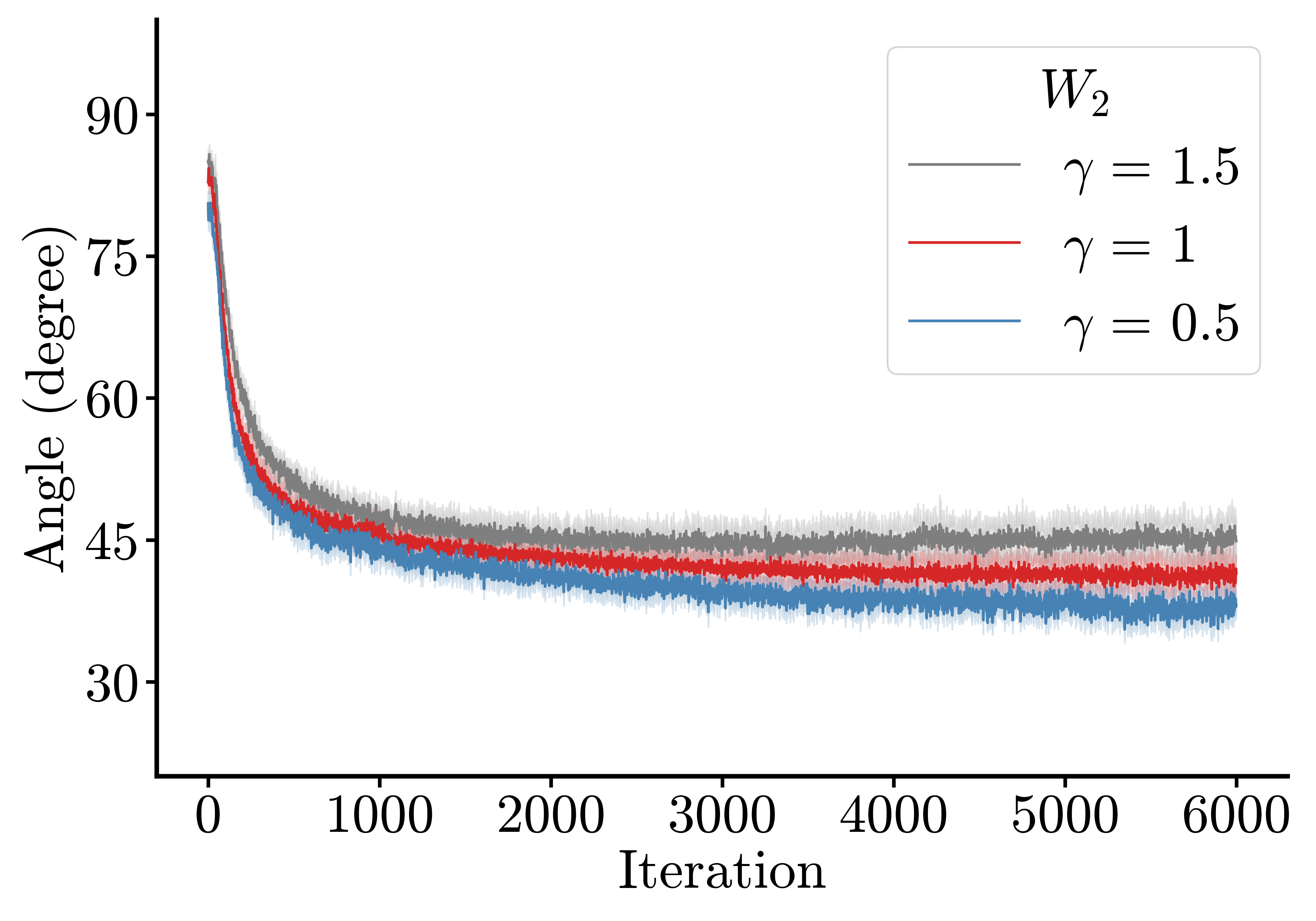}
        \caption{}
        \label{fig:alignment_w2}
    \end{subfigure}
\vskip -3mm
    \caption{Evolution of alignment during training: (a) Alignment between gradient directions from FTP and BP with $\gamma = 1$; (b) Alignment between products of forward weight matrices  (\(\mW_2^T \mW_3^T \)) and the feedback matrix $\mG$; (c,d) Alignment of gradients for $\mW_1$ and $\mW_2$ under varying $\gamma$.}
    \label{fig:combined}
    \vspace{-5mm}
\end{figure*}

To understand how the magnitude of the error signal affects alignment, we introduce a hyperparameter \( \gamma \) that scales the FTP target as 
\begin{equation}
\boldsymbol{\tau}_1 = \gamma(\sigma(\mG \vy) - \sigma(\mG \vh_L)) + \vh_1 
\end{equation}
Setting \( \gamma = 1 \) yields the standard FTP formulation. We evaluated its effect over the range [0.1, 1.5], and here we report results for \( \gamma = 0.5 \) and \( \gamma = 1.5 \) in Figure~\ref{fig:alignment_w1} and ~\ref{fig:alignment_w2} to highlight the impact of error signal scaling. Each setting was evaluated over 10 random seeds; solid lines show the mean and shaded regions denote $\pm 1$ standard deviation. When \( \gamma = 0.5 \), we observe stronger alignment in both hidden layers, compared to the case of \(\gamma = 1\). Conversely, increasing \( \gamma \) to $1.5$ degrades the alignment. These results indicate that \( \gamma \) acts as a regularizer, which governs error-signal strength and directional alignment with BP. When \( \gamma \) is large, the FTP target shift becomes overly aggressive, potentially pushing the hidden representations too far from their initial state, which can destabilize learning and lead to misaligned updates. In contrast, a smaller \( \gamma \) produces more conservative target shifts, allowing the network to adjust gradually in directions that are more naturally aligned with the forward pathway and ultimately closer to the BP gradients.

In addition to gradient alignment, we analyze the structural alignment between the product of forward weights (i.e., \(\mW_2^T \mW_3^T \)) and the projection matrix \(\mG \), following FA \citep{lillicrap2016random}. We flatten them and compute the angle between the flattened vectors throughout training, which is depicted in Figure~\ref{fig:alignment_wg}. Initially, the angle is close to \( 90^\circ \), reflecting a random orientation. Over time, the angle steadily decreases, indicating that FTP promotes consistent forward–backward correlations. This behavior contrasts with the anti-alignment in PEPITA~\citep{pepita_della}, and supports the hypothesis \citep{lillicrap2016random} that biologically plausible learning rules such as FTP naturally give rise to structured and coordinated representations.

\subsection{Performance of FTP in Emerging Hardware and Edge Devices}
\label{sec:robust_FTP}
BP relies on symmetry between forward and backward weights, which is difficult to maintain on emerging analog accelerators such as RRAM and other non-volatile memories, due to non-idealities including noise, programming errors, and variability. These asymmetries disrupt gradient flow and substantially degrade BP performance under low-bit precision. We demonstrate in Appendix~\ref{appendix_weightsymmtery_bp} how the performance of BP is susceptible to such non-ideal conditions. 

Programming errors, a key source of non-ideality, accumulate differently in FTP and BP. In FTP, the feedback matrix \(\mG \) is randomly initialized and programmed only once at the start of training. Any programming error introduced then remains embedded in \(\mG \) throughout the training process. Since FTP reuses the same forward matrices for both activations and layer-wise targets, any programming error in matrices affects both pathways in a consistent manner, making it inherently robust to hardware-induced asymmetry and stable even in noisy, low-precision settings. In contrast, BP requires continual reprogramming of forward and backward matrices--unless complicated bidirectional peripheral circuitry is implemented~\citep{Cross_Wan2022}. Independent reprogramming of forward and backward matrices often violates the symmetry, especially in analog crossbars (e.g., RRAM, PCM), where these two are stored in physically different memory cells \citep{Yi2023}. Here, \emph{symmetry} refers to the proportion of elements in the backward weight matrix that remain uncorrupted relative to the forward matrix. Reduced symmetry means backward weights increasingly deviate from forward-transpose, introducing new, uncorrelated errors at every step.
\begin{figure*}[ht!]
\vspace{-3mm}
\centering
\begin{subfigure}{0.47\textwidth}
    \vskip -.1in
    \centering
    \includegraphics[width=1\linewidth]{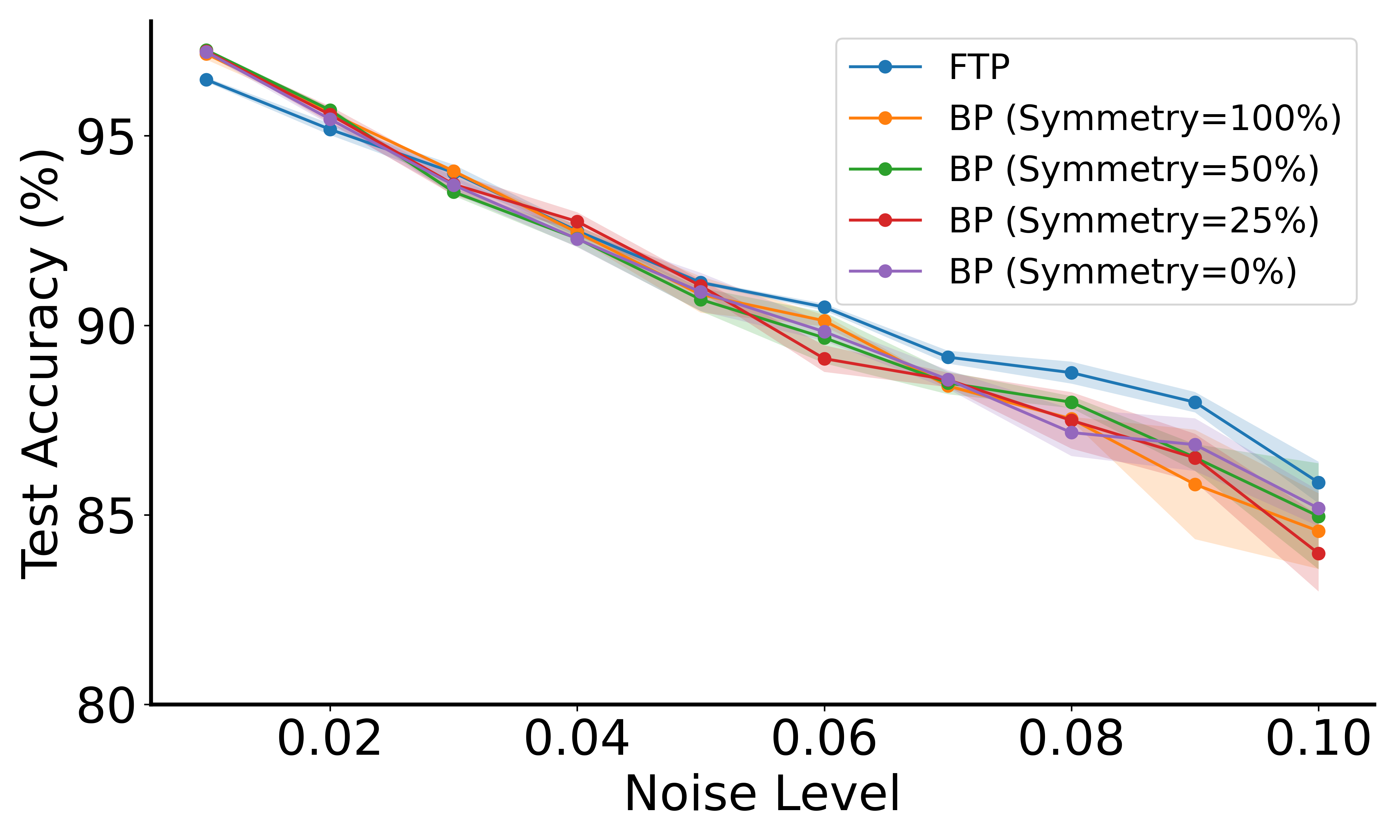}  
    \caption{} 
    \label{fig:alignment_1}
\end{subfigure}
\hfill
\begin{subfigure}{0.47\textwidth}
    \centering
    \includegraphics[width=1\linewidth]{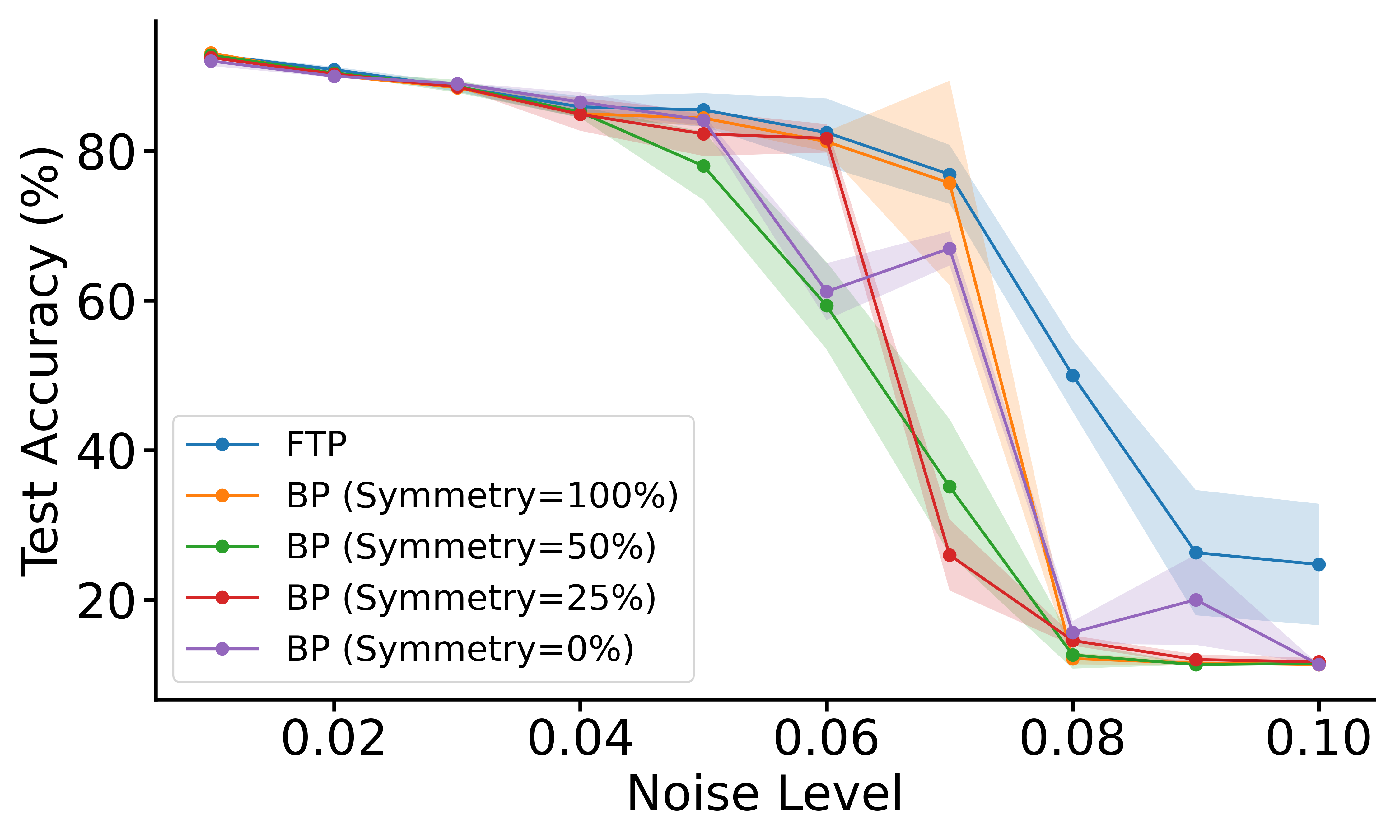}  
    \caption{} 
    \label{fig:alignment_2}
\end{subfigure}
\vskip -3mm
\caption{Performance of FTP and BP when programming errors are considered for (a) 8-bit and (b) 4-bit precision devices.}
\label{fig:programming_error}
\vspace{-5mm}
\end{figure*}


To evaluate this effect, we injected additive Gaussian noise into weights to simulate programming errors. The noise level, shown on the x-axis of Figure~\ref{fig:programming_error}, is parameterized by \( \alpha \), scaling the standard deviation of noise distribution as \( \text{std} = \alpha \times |w| \), where \( w \) is the weight being corrupted \citep{nonidealrram}. Thus, larger weights incur proportionally larger noise, with higher \( \alpha \) producing stronger perturbations. We also compare against FA, which uses fixed random feedback matrices and avoids reprogramming backward paths in hardware, making it a relevant robustness baseline (see Appendix~\ref{robustness_appendix} for detailed discussion). We also show robustness of FTP compared to other forward learning methods in Appendix~\ref{appendix_forward_baselines}. As shown in Figure~\ref{fig:programming_error} and Appendix~\ref{robustness_appendix}, FTP consistently shows improved robustness under programming errors, especially as noise levels increase. Notably, in 4-bit systems, FTP maintains higher accuracy than BP at high noise levels, underscoring its robustness to programming errors in low-precision and non-ideal hardware.

\subsection{Hardware-Aware Analysis of FTP’s Efficiency}

To assess FTP’s suitability for resource-constrained environments such as TinyML applications, we estimated the number of MAC operations based on dataset sizes and epoch counts aligned with the MLCommons/Tiny Benchmark \citep{mlcommons2023github}, following a methodology similar to \citet{ff_pepita_mlcommons}. The MAC counts are analytical estimates derived from network topology, dataset sizes, and epoch counts; no model in Table~\ref{model-dataset-table} was trained, and no accuracy is reported for it. This estimation offers a hardware-aware perspective on training efficiency under realistic deployment constraints. As shown in the table, FTP consistently achieves comparable MAC counts relative to BP and substantially lower computational overhead compared to other forward-only methods such as FF, PEPITA, and DTP. We report MAC counts as a proxy for arithmetic workload rather than energy, since memory access and data movement often dominate total energy on real accelerators; end-to-end energy measurement on physical hardware is left to future work. However, FTP's data-access pattern is broadly similar to BP's: to update a given layer, both read that layer's stored activation, the incoming activation from the layer below, and the feedforward weight matrix, and the per-layer target FTP propagates forward carries the same read cost as the per-layer error BP propagates backward. Because FTP reuses these feedforward weights in their original orientation, it also avoids the transposed-weight reads and backward-matrix reprogramming that BP requires in case of analog accelerators. Thus the arithmetic savings of FTP are unlikely to be offset by added data movement. Analytically, FTP's update pattern also admits pipelined scheduling, which would reduce wall-clock training time relative to BP's update-locked schedule; we sketch this argument in Appendix \ref{appendix_memory}. Taken together with FTP's robustness under simulated noise, these estimates position FTP as a promising candidate for on-device learning in resource-constrained systems.



\begin{table*}[h]
\vspace{-2mm}
\centering
\caption{Comparison of MAC (millions) for TinyML datasets with various algorithms and percentage change in MAC (\%) with respect to BP. }
\label{model-dataset-table}
\vspace{-3mm}
\small
\resizebox{\textwidth}{!}{%
\begin{tabular}{lcccccccc}
\toprule
\multirow{2}{*}{Learning Method} & \multicolumn{2}{c}{DS-CNN/SC} & \multicolumn{2}{c}{MobileNet/VWW} & \multicolumn{2}{c}{ResNet/CIFAR10} & \multicolumn{2}{c}{AE/ToyADMOS} \\
\cmidrule(lr){2-3} \cmidrule(lr){4-5} \cmidrule(lr){6-7} \cmidrule(lr){8-9}
& MAC & MAC (\%) & MAC & MAC (\%) & MAC & MAC (\%) & MAC & MAC (\%) \\
\midrule
BP      & 7.70  & 0.00\%   & 22.40 & 0.00\%   & 37.10 & 0.00\%   & 0.73 & 0.00\%   \\
FF      & 10.98 & 42.58\%  & 31.52 & 40.70\%  & 50.42 & 35.89\%  & 1.09 & 49.51\%  \\
PEPITA  & 8.03  & 4.25\%   & 22.95 & 2.47\%   & 37.60 & 1.35\%   & 1.23 & 69.10\%  \\
DTP     & 17.15 & 122.70\% & 50.78 & 126.70\% & 76.46 & 106.10\% & 1.49 & 104.00\% \\
\textbf{FTP} & \textbf{7.89} & \textbf{2.51\%} & \textbf{22.47} & \textbf{0.33\%} & \textbf{37.46} & \textbf{0.96\%} & \textbf{0.90} & \textbf{23.03\%} \\
\bottomrule
\end{tabular}%
}
\vskip -0.2in
\end{table*}

 \section{Conclusion}
We have introduced Forward Target Propagation (FTP), a biologically plausible and computationally lightweight learning algorithm that serves as a forward-only alternative to backpropagation. Across fully connected, convolutional, and recurrent architectures, FTP matches or outperforms existing forward-only and target-propagation methods while remaining close to BP on dense and recurrent tasks, at a substantially lower computational cost than the other bio-plausible algorithms. A key feature of FTP is its ability to assign global credit through purely local losses, thereby eliminating the need for weight symmetry, a non-local objective, and backward signal propagation. Among recent bio-inspired and forward-only learning algorithms, our results show that FTP delivers better accuracy with more reliable depth scaling and greater computational efficiency, yielding lower multiply-and-accumulate (MAC) operations than FF, PEPITA, and DTP. FTP also exhibits resilience to low-bit precision and noisy hardware conditions, where conventional BP tends to degrade. These advantages position FTP as a prospective solution for deployment in TinyML and neuromorphic systems. While further development is needed to fully realize its potential at scale, our findings establish FTP as a robust, computationally efficient, and biologically plausible learning framework for embedded AI systems and neuromorphic computing.




\section*{Acknowledgments}
This work was supported by the U.S. National Science Foundation (NSF) under Award No.~ECCS-2432082, jointly funded by the Directorate for Engineering, Division of Electrical, Communications and Cyber Systems, and the Established Program to Stimulate Competitive Research (EPSCoR). N.S.A.\ and S.Y.\ were additionally supported by the Sandia Laboratory-Directed Research and Development (LDRD) Program. Sandia National Laboratories is a multimission laboratory managed and operated by National Technology and Engineering Solutions of Sandia, LLC, a wholly owned subsidiary of Honeywell International, Inc., for the U.S. Department of Energy's National Nuclear Security Administration under contract DE-NA0003525. Any opinions, findings, and conclusions or recommendations expressed in this material are those of the authors and do not necessarily reflect the views of the National Science Foundation, the U.S. Department of Energy, or the United States Government.

\bibliography{main}
\bibliographystyle{tmlr}

\appendix

\newpage

\section{Limitations and Future Work}
\label{sec:limitations}

FTP is designed for supervised tasks, where a target is defined at the output layer; it is not a self-supervised rule and does not construct its own objective, though it applies without modification to any label-free objective that supplies such a target, such as reconstruction,
which is a natural extension to explore (demonstrated in Appendix~\ref{app:ae}). Within this setting, several limitations bound our claims. Our evaluation covers small- to moderate-scale architectures, with depth verified from three to six layers (Appendix~\ref{scalability_appendix}); we therefore do not claim
scalability to deep standard or residual networks, on which no forward-only rule has yet been demonstrated \citep{pepita_della, pepita_srini, counter_current, dendritic_ll, Hinton2022TheFA}, and a full-precision gap to BP remains on convolutional neural networks.
Extending FTP to deeper models, larger datasets, and attention-based architectures, where projecting the output target through the attention mechanism is the open problem, with small-scale language modelling against BP-trained baselines as the appropriate first test, is left to future work. In the recurrent setting, FTP reduces but does not eliminate the storage of intermediate temporal activations, retaining a per-timestep eligibility trace while removing BPTT's reverse-time traversal and weight transport. Our efficiency and robustness evidence rests on Multiply-and-Accumulate (MAC) counts and simulated low-precision noise rather than physical hardware, so we treat energy and runtime benefits as prospective; end-to-end measurement on instrumented
hardware remains future work. Finally, our alignment analysis is theoretically limited to a single-sample linear setting, which we discuss further in Appendix~\ref{appendix_theory}.

\section{Broader Impact of Our Work}
\label{appendix_impact}
Energy consumption is an increasingly central concern in machine learning, particularly for edge and on-device deployment where power budgets are tight. Our work proposes Forward Target Propagation (FTP), a biologically inspired, forward-only learning algorithm that offers significant advantages in robustness and hardware compatibility compared to standard backpropagation (BP). These characteristics make FTP particularly relevant for enabling sustainable training and deployment of large-scale models in both cloud and edge settings. Below, we outline key broader impacts of our approach:

\textbf{Compatibility with Quantization and Noisy Hardware:} Emerging memory technologies such as Resistive Random Access Memory (RRAM) and Phase-Change Memory (PCM) offer compact, analog, and energy-efficient platforms for neural computation, but suffer from non-idealities such as stochastic noise, variability and limited precision. Unlike BP, which requires high-precision symmetric weight updates, FTP is inherently robust to such hardware imperfections due to its forward-only and fixed-feedback structure. Our simulations indicate that FTP retains its performance better than BP under low-bit precision and noisy conditions, making it particularly well suited for deployment on in-memory computing systems.

\textbf{Enabling Sustainable, On-Device Intelligence:} By aligning algorithmic design with the constraints and strengths of emerging and edge hardware, FTP opens the door to efficient on-device learning in applications such as autonomous sensors, wearable devices, and edge robotics. We stress that these are potential benefits rather than demonstrated ones: our experiments establish reduced MAC counts and robustness to simulated noise, not a measured reduction in energy consumption. Confirming an energy or environmental benefit would require deployment on physical hardware with power instrumentation, which we identify as an important direction for future work.

\section{Biological Properties of FTP}
Table~\ref{tab:learning_rules_comparison} compares various learning rules across key criteria relevant to biological plausibility. A checkmark (\ding{51}) indicates that the corresponding issue is addressed (i.e., the property is satisfied), while a cross (\textcolor{red}{\ding{55}}) indicates that the issue remains unsolved. The tilde symbol (\textcolor{orange}{$\sim$}) denotes a partial solution. Notably, FTP satisfies all the listed criteria, with a partial resolution of update locking.

\begin{table*}[h]
\caption{Comparison of learning rules based on key metrics.}
\label{tab:learning_rules_comparison}
\begin{tabular}{lccccc}
\toprule
\small  Learning Rule & \small Forward Only & \small Weight Symmetry & Local & \small  Activity Freezing & \small Update Unlocked \\
\midrule
BP      & \textcolor{red}{\ding{55}} & \textcolor{red}{\ding{55}} & \textcolor{red}{\ding{55}} & \textcolor{red}{\ding{55}} & \textcolor{red}{\ding{55}} \\
DTP     & \textcolor{red}{\ding{55}} & \ding{51} & \ding{51} & \textcolor{red}{\ding{55}} & \textcolor{red}{\ding{55}} \\
PEPITA  & \ding{51} & \ding{51} & \ding{51} & \ding{51} & \textcolor{orange}{$\sim$}   \\
FF      & \ding{51} & \ding{51} & \ding{51} & \ding{51} & \ding{51} \\
FTP     & \ding{51} & \ding{51} & \ding{51} & \ding{51} & \textcolor{orange}{$\sim$}    \\
\bottomrule
\end{tabular}
\end{table*}

\section{Analysis of FTP in Network with Two Hidden Layers}
\label{appendix_theory}
We consider a neural network with two hidden layers with activation $\sigma(\cdot)$. In \cref{sec:alignment_BP}, we prove that (under some conditions) the weight update direction for the hidden layers in FTP is within $90\degree$ of that of BP. With further restricted assumptions, we show in \cref{sec:alignment_GN}, that the update direction for the hidden neurons in FTP is aligned with the Gauss-Newton direction. We followed the structure of the proofs from \citep{lillicrap2016random}.

For an input $\vx \in \sR^{d_x}$, we have
\begin{align}
    \vh_1 &= \sigma(\mW_1\vx) \\
    \vh_2 &= \sigma(\mW_2\vh_1) \\
    \vh_3 &= \sigma(\mW_3\vh_2) 
\end{align}

where $\vh_3 \in \sR^{d_y}$ is the output, and $\vh_1 \in \sR^{d_1}, \vh_2 \in \sR^{d_2}$ are hidden neurons. In forward target propagation (FTP), the losses to train the weight matrices $\mW_1, \mW_2$ and $\mW_3$ are as follows:
\begin{align}
    \mathcal{L}_1 &= \frac{1}{2} ||\operatorname{stop\_grad}(\boldsymbol{\tau}_1) - \vh_1||^2 \\
    \mathcal{L}_2 &= \frac{1}{2} ||\operatorname{stop\_grad}(\boldsymbol{\tau}_2) - \vh_2||^2 \\
    \mathcal{L}_3 &= \frac{1}{2} ||\vy - \vh_3||^2 
\end{align}
where $\vy \in \sR^{d_y}$ is the target value at the output layer and $\boldsymbol{\tau}_1, \boldsymbol{\tau}_2$ are target values (suggested by FTP) at the hidden layers according to the following.
\begin{align}
    \boldsymbol{\tau}_1 &= \vh_1 + \sigma(\mG\vy) - \sigma(\mG\vh_3) \label{eqn: target_h1}\\
    \boldsymbol{\tau}_2 &= \sigma(\mW_2\boldsymbol{\tau}_1) \label{eqn: target_h2}
\end{align}
$\mG$ is a $d_1 \times d_y$ projection matrix where $\emG_{i,j}$ can be distributed according to $\mathcal{N}(0,1)$. Because $\boldsymbol{\tau}_1$ and $\boldsymbol{\tau}_2$ are detached via $\operatorname{stop\_grad}$, every gradient below is taken only with respect to the activation $\vh_i$, even though both $\vh_i$ and $\boldsymbol{\tau}_i$ depend on $\mW_i$; consequently $\tfrac{1}{2}\lVert\operatorname{stop\_grad}(\boldsymbol{\tau}_i)-\vh_i\rVert^2$ and $\tfrac{1}{2}\lVert\vh_i-\operatorname{stop\_grad}(\boldsymbol{\tau}_i)\rVert^2$ yield the same local update.

Note that we considered mean squared error as the loss $\mathcal{L}_3$ at the output layer for our theoretical analysis. Furthermore, the other assumptions, e.g., a single sample with linear networks, and specific initialization of the weight matrices, are considerably more restrictive than the conditions of our empirical experiments. Hence, the analysis presented here does not provide a complete theoretical evidence for our empirical results across multiple models and datasets.

\subsection{Alignment in Weight Update Direction Between FTP and BP} \label{sec:alignment_BP}
In this section, we consider the same neural network described above, except for linear activation functions. 

Denoting error signal as $\ve = \vy-\vh_3$, we have the following gradient direction for each weight matrix.
\begin{align}
    \pdv{\mathcal{L}_3}{\mW_3} &= \tcb{-}\ve\vh_2^T\\
    \pdv{\mathcal{L}_2}{\mW_2} &= \tcb{-}(\boldsymbol{\tau}_2 - \vh_2) \vh_1^T \\
    \pdv{\mathcal{L}_1}{\mW_1} &= \tcb{-}(\boldsymbol{\tau}_1 - \vh_1) \vx^T
\end{align}
Note the FTP has the same gradient for $\mW_3$ as in BP. 
From \cref{eqn: target_h1,eqn: target_h2}, we have
\begin{align*}
    \boldsymbol{\tau}_1 -  \vh_1 &= \mG\vy - \mG\vh_3 = \mG\ve \\
    \boldsymbol{\tau}_2 - \vh_2 &= \mW_2\boldsymbol{\tau}_1 - \mW_2\vh_1 = \mW_2(\boldsymbol{\tau}_1 - \vh_1) = \mW_2\mG\ve
\end{align*}
Finally, we get the gradient direction for $\mW_1, \mW_2$ under FTP as:
\begin{align}
    \pdv{\mathcal{L}_2}{\mW_2} &= \tcb{-}\mW_2\mG\ve \vh_1^T  \\
    \pdv{\mathcal{L}_1}{\mW_1} &= \tcb{-}\mG\ve \vx^T
\end{align}

For BP,
\begin{align}
    \pdv{\mathcal{L}_2}{\mW_2} &= \tcb{-}\mW_3^T\ve \vh_1^T \\
    \pdv{\mathcal{L}_1}{\mW_1} &= \tcb{-}\mW_2^T\mW_3^T\ve \vx^T
\end{align}

We hypothesize the following alignment between FTP and BP's gradient direction for any non-zero $\ve$:
\begin{align}
    \langle \mG\ve, \mW_2^T\mW_3^T\ve \rangle > 0 \\
    \langle \mW_2\mG\ve, \mW_3^T\ve \rangle > 0
\end{align}
\\
\begin{lemma} \label{lemma_1}
If we initialize $\mW_1$ and  $\mW_3$ with zero entries and $\mW_2$ being initialized with $\mA$, then there exist some scalars $s_1, s_{\mW_1}, s_{\mW_2}, s_{\mW_3}$
at every time step (during training using FTP) such that,
\begin{align}
    \vh_1 &= s_1\mG\vy  \label{eqn:lemma_1_1}\\
    \mW_1 &= s_{\mW_1}\mG\vy\vx^T \label{eqn:lemma_1_2}\\
    \mW_2 &= \mA(\mI + s_{\mW_2}\mG\vy(\mG\vy)^T) \label{eqn:lemma_1_3}\\
    \mW_3 &= s_{\mW_3}\vy(\mA\mG\vy)^T \label{eqn:lemma_1_4}
\end{align}
\end{lemma}

\begin{proof}

Here, we provide the proof by induction. At initialization ($t=0$) \cref{eqn:lemma_1_1,eqn:lemma_1_2,eqn:lemma_1_3,eqn:lemma_1_4} satisfies with $s_1=s_{\mW_1}= s_{\mW_2}= s_{\mW_3}=0$. Now, if this is true for any $t>0$, then we need to prove this holds for $t+1$.

\begin{align}
    \vh_3^{(t)} &= \mW_3^{(t)}\vh_2^{(t)} \\
    &= s_1^{(t)}s_{\mW_3}^{(t)}\vy(\mA\mG\vy)^T \mW_2^{(t)}\mG\vy\\
    &= s_1^{(t)} s_{\mW_3}^{(t)}\vy (\mA\mG\vy)^T \mA(\mI + s_{\mW_2}^{(t)}\mG\vy(\mG\vy)^T) \mG \vy
\end{align}

\begin{align*}
    (\mA\mG\vy)^T \mA(\mI + s_{\mW_2}^{(t)}\mG\vy(\mG\vy)^T) \mG \vy &= (\mG\vy)^T \mA^T\mA(\mI + s_{\mW_2}^{(t)}\mG\vy(\mG\vy)^T) \mG \vy \\
    &= ||\mA\mG\vy||^2 +  s_{\mW_2}^{(t)}||\mA\mG\vy||^2||\mG\vy||^2
\end{align*}

Now, 

\begin{align}
    \vh_3^{(t)} 
    &= s_1^{(t)} s_{\mW_3}^{(t)} \left( ||\mA\mG\vy||^2 +  s_{\mW_2}^{(t)}||\mA\mG\vy||^2||\mG\vy||^2\right)\vy \\
    &= s_3^{(t)}\vy \quad (\text{denoting } s_3^{(t)} = s_1^{(t)} s_{\mW_3}^{(t)}\left( ||\mA\mG\vy||^2 +  s_{\mW_2}^{(t)}||\mA\mG\vy||^2||\mG\vy||^2\right)  )\\
    \ve^{(t)} &= \vy - \vh_3^{(t)} = (1-s_3^{(t)})\vy \label{eqn: lemma_1_error}
\end{align}

\begin{align*}
    \mW_1^{(t+1)} &= \mW_1^{(t)} + \eta_1 \mG\ve^{(t)}\vx^T \\
    &= s_{\mW_1}^{(t)}\mG\vy\vx^T + \eta_1\mG(1-s_3^{(t)})\vy\vx^T\\
    &= s_{\mW_1}^{(t+1)}\mG\vy\vx^T  \quad \left(\text{where }s_{\mW_1}^{(t+1)} = s_{\mW_1}^{(t)} + \eta_1(1-s_3^{(t)}) \right)
\end{align*}

\begin{align*}
    \mW_2^{(t+1)} &= \mW_2^{(t)} + \eta_2 \mW_2^{(t)}\mG\ve^{(t) }{\vh_1^T}^{(t)} \\
    &= \mA + s_{\mW_2}^{(t)}\mA\mG\vy{(\mG\vy)}^{T} + \eta_2 \left(\mA + s_{\mW_2}^{(t)}\mA\mG\vy{(\mG\vy)}^{T}\right) \mG (1-s_3^{(t)})\vy {(s_1^{(t)}\mG\vy)}^{T} \\
    &= \mA + (s_{\mW_2}^{(t)}+  \eta_2 s_1^{(t)}(1-s_3^{(t)})) \mA\mG\vy{(\mG\vy)}^{T} + \eta_2 s_1^{(t)} s_{\mW_2}^{(t)} (1-s_3^{(t)}) \mA\mG\vy{(\mG\vy)}^{T} \mG\vy{(\mG\vy)}^{T} \\
    &= \mA + (s_{\mW_2}^{(t)}+  \eta_2 s_1^{(t)} (1-s_3^{(t)})) \mA\mG\vy{(\mG\vy)}^{T} + \eta_2 s_1^{(t)} s_{\mW_2}^{(t)}  (1-s_3^{(t)}) \mA\mG\vy||\mG\vy||^2 {(\mG\vy)}^{T} \\
    &= \mA +  s_{\mW_2}^{(t+1)} \mA\mG\vy{(\mG\vy)}^{T} 
\end{align*}

Here, $ s_{\mW_2}^{(t+1)} = s_{\mW_2}^{(t)}+  \eta_2 s_1^{(t)} (1-s_3^{(t)}) + \eta_2 s_1^{(t)} s_{\mW_2}^{(t)}  (1-s_3^{(t)}) ||\mG\vy||^2$

\begin{align*}
    \mW_3^{(t+1)} &= \mW_3^{(t)} + \eta_3 \ve^{(t)}{\vh_2^T}^{(t)}\\
    &= s_{\mW_3}^{(t)}\vy(\mA\mG\vy)^T + \eta_3 s_1^{(t)} (1-s_3^{(t)})\vy(\mW_2^{(t)}\mG\vy)^T \\
    &= s_{\mW_3}^{(t)}\vy(\mA\mG\vy)^T + \eta_3 s_1^{(t)} (1-s_3^{(t)})\vy(\mA(\mI + s_{\mW_2}^{(t)}\mG\vy(\mG\vy)^T)\mG\vy)^T \\ 
    &= s_{\mW_3}^{(t)}\vy(\mA\mG\vy)^T + \eta_3 s_1^{(t)} (1-s_3^{(t)}) \left(1 + s_{\mW_2}^{(t)}||\mG\vy||^2\right) \vy(\mA\mG\vy)^T \\
    &= s_{\mW_3}^{(t+1)}\vy(\mA\mG\vy)^T \quad (\text{where }s_{\mW_3}^{(t+1)} =  s_{\mW_3}^{(t)}+\eta_3 s_1^{(t)} (1-s_3^{(t)})\left(1 + s_{\mW_2}^{(t)}||\mG\vy||^2\right) )
\end{align*}

\begin{align*}
    \vh_1^{(t+1)} &= \mW_1^{(t+1)}\vx \\
    &= \left(\mW_1^{(t)} +\eta_1 \mG\ve^{(t)}\vx^T\right)\vx \\
    &= \vh_1^{(t)} + \eta_1 (1-s_3^{(t)})\mG\vy||\vx||^2 \\
    &= s_1^{(t)}\mG\vy + \eta_1 (1-s_3^{(t)}) ||\vx||^2 \mG\vy \\
    &= s_1^{(t+1)}\mG\vy \quad (\text{where } s_1^{(t+1)} =  s_1^{(t)} + \eta_1 (1-s_3^{(t)}) ||\vx||^2 )
\end{align*}
\end{proof}

\begin{remark} \label{remark_1}
Initializing the scalars $s_1=s_{\mW_1}=s_{\mW_2}=s_{\mW_3}=0$ at $t=0$, they stay positive from $t=2$ by appropriately selecting learning rates $\eta_i > 0$ such that $s_3^{(t)}\leq1$.
\end{remark}

\begin{theorem}
Under the same conditions in \cref{lemma_1}, subject to the positivity of scalars $s_1,s_{\mW_1},s_{\mW_2},s_{\mW_3}$ (\cref{remark_1}), the FTP and BP's  (non-zero) gradient direction for $\mW_1,\mW_2$ are within $90\degree$ of each other, i.e. 
\begin{align}
    \langle \mG\ve, \mW_2^T\mW_3^T\ve \rangle > 0 \label{eqn: thm_1_w_1}\\
    \langle \mW_2\mG\ve, \mW_3^T\ve \rangle > 0 \label{eqn: thm_1_w_2}
\end{align}
\end{theorem}
\begin{proof}

Note  $\langle \mG\ve, \mW_2^T\mW_3^T\ve \rangle = \langle \mW_2\mG\ve, \mW_3^T\ve \rangle$

From \cref{eqn:lemma_1_3,eqn:lemma_1_4} of \cref{lemma_1}, we have
\begin{align}
    \mW_3\mW_2 &= s_{\mW_3}\vy(\mA\mG\vy)^T \mA(\mI + s_{\mW_2}\mG\vy(\mG\vy)^T)  \nonumber \\ 
    &= s_{\mW_3}\vy(\mG\vy)^T\mA^T\mA (\mI + s_{\mW_2}\mG\vy(\mG\vy)^T) \label{eqn: W3W2}  \\
    &= s_{\mW_3}\vy(\mG\vy)^T\mA^T\mA + s_{\mW_3}s_{\mW_2}||\mA\mG\vy||^2\vy(\mG\vy)^T \nonumber \\
    &= s_{\mW_3}\vy(\mG\vy)^T\mA^T\mA + s_{3,2}\vy(\mG\vy)^T \quad (\text{denoting } s_{3,2} = s_{\mW_3}s_{\mW_2}||\mA\mG\vy||^2 ) \nonumber
\end{align}
Now for the weight matrix of the first hidden layer, $\mW_1$
\begin{align*}
    \langle \mG\ve, \mW_2^T\mW_3^T\ve \rangle &=  (\mG\ve)^T (\mW_3\mW_2)^T\ve  \\
    &= s_{\mW_3} (1-s_3)^2 (\mG\vy)^T ( \vy(\mG\vy)^T\mA^T\mA)^T \vy +s_{3,2} (1-s_3)^2 (\mG\vy)^T ( \vy(\mG\vy)^T)^T \vy \\
    &= s_{\mW_3} (1-s_3)^2 ||\mA\mG\vy||^2||\vy||^2+s_{3,2} (1-s_3)^2 ||\mG\vy||^2||\vy||^2 > 0 \quad \forall \vy \neq \mathbf{0} 
\end{align*}
The last inequality follows from $s_{\mW_2}, s_{\mW_3}$ being  positive scalars (\cref{remark_1}).  
\end{proof}

\subsection{Relation with Gauss-Newton Update for Hidden Neurons}\label{sec:alignment_GN}

In linear neural network with two hidden layers, the feedback signal (under FTP) for update in $\vh_1$ is $\mG\ve$, and we postulate that it is aligned with the Gauss-Newton update signal, $(\mW_3\mW_2)^\dagger$. 
If we can show there exists a relation such as $s\mG\ve = (\mW_3\mW_2)^\dagger \ve$ with a positive scalar $s$, then we can conclude update direction for the hidden neurons $\vh_1$ is aligned with the Gauss-Newton direction.

\begin{theorem}
Under the same conditions in \cref{lemma_1}, subject to the positivity of scalars $s_1,s_{\mW_1},s_{\mW_2},s_{\mW_3}$ (\cref{remark_1}), and $\mW_2$ being initialized with $\mA$ such that $\mA^T\mA = \mI$ (requires $d_2\geq d_1$), there exists a positive scalar $s$ such that 
\begin{align}
    s\mG\ve = (\mW_3\mW_2)^\dagger \ve
\end{align}
\end{theorem}

\begin{proof}

From \cref{eqn: W3W2} with $\mA^T\mA =\mI$, 
\begin{align*}
    \mW_3\mW_2 &= s_{\mW_3}\vy(\mG\vy)^T  + s_{\mW_3}s_{\mW_2}||\mG\vy||^2\vy(\mG\vy)^T\\
    &= s'_{3,2}\vy(\mG\vy)^T \quad (\text{where }s'_{3,2} =   s_{\mW_3} +s_{\mW_3}s_{\mW_2}||\mG\vy||^2 )
\end{align*}

Showing $s\mG\ve = (\mW_3\mW_2)^\dagger \ve$ is equivalent to 
$s\mG\vy = (\vy(\mG\vy)^T)^\dagger \vy$ since we have $\ve = (1-s_3)\vy$ from \cref{eqn: lemma_1_error}, and $s_{\mW_2}, s_{\mW_3}$ are positive (\cref{remark_1}).

\begin{align*}
    (\vy(\mG\vy)^T)^\dagger\vy  &= 
     (\vy^T\mG^T)^\dagger\vy^\dagger\vy \\
    &=  (\mG^T)^\dagger (\vy^{T})^{\dagger} \vy^\dagger\vy\\
    &=  (\mG^T)^\dagger (\vy^{T})^{\dagger} \vy^{T}(\vy^{T})^{\dagger} \\
    &=  (\mG^T)^\dagger(\vy^{T})^{\dagger} \\
    &=  ((\mG\vy)^T)^{\dagger} \\
    &=  \mG\vy ((\mG\vy)^T\mG\vy)^{-1}\\
    &=  ||\mG\vy||^{-2} \mG\vy \\
    &= s\mG\vy
\end{align*}
\end{proof}

\begin{remark}
When $d_2\geq d_1$, we can construct $\mA$ 
by selecting $d_1$ orthonormal column vectors such that $\mA^T\mA=\mI$. This adds further restriction to the initialization scheme mentioned in \cref{lemma_1}.
\end{remark}

\section{Relation to Indirect Feedback Alignment}
\label{app:ifa}

Indirect Feedback Alignment (IFA) was introduced alongside Direct Feedback Alignment by
\citet{nokland2016direct}. Like DFA, it injects the output error through a fixed random
matrix; unlike DFA, it does so only at the first hidden layer and then propagates the signal
upward through the forward weights. In this section we show that FTP and IFA are
equivalent under linear activations and characterize how they differ otherwise.

If we consider a network of depth $L$ with $\va_i=\mW_i\vh_{i-1}$, $\vh_i=\sigma(\va_i)$, and
$\vh_0=\vx$. Following Appendix~\ref{appendix_theory}, let $\ve=\vy-\vh_L$. IFA computes
\begin{align}
    \delta_1 &= (\mB\ve)\odot\sigma'(\va_1), \qquad
    \delta_i = (\mW_i\delta_{i-1})\odot\sigma'(\va_i),\quad 1<i<L,
\end{align}
with updates $\Delta\mW_i^{\mathrm{IFA}}=\eta\,\delta_i\vh_{i-1}^\top$ for $i<L$ and
$\Delta\mW_L^{\mathrm{IFA}}=\eta\,\ve\,\vh_{L-1}^\top$, where $\mB\in\sR^{d_1\times d_y}$ is
fixed and random.
FTP uses
$\Delta\mW_i^{\mathrm{FTP}}=\eta\,\big[(\boldsymbol{\tau}_i-\vh_i)\odot\sigma'(\va_i)\big]\vh_{i-1}^\top$
(Eq.~\ref{eqn: w_update}), with $\boldsymbol{\tau}_1=\vh_1+\sigma(\mG\vy)-\sigma(\mG\vh_L)$ and
$\boldsymbol{\tau}_i=\sigma(\mW_i\boldsymbol{\tau}_{i-1})$ for $1<i<L$. The output layer is
updated identically by FTP, IFA, and BP.

Let $\sigma(z)=z$, so that $\sigma'\equiv1$, and let the IFA feedback matrix be $\mB=\mG$. The
IFA recursion then reduces to $\delta_1=\mG\ve$ and $\delta_i=\mW_i\delta_{i-1}$ for
$1<i<L$. 

For FTP, the first-layer target difference is
$\boldsymbol{\tau}_1-\vh_1=\mG\vy-\mG\vh_L=\mG\ve=\delta_1$, and for $1<i<L$ the second
forward pass gives
$\boldsymbol{\tau}_i-\vh_i=\mW_i\boldsymbol{\tau}_{i-1}-\mW_i\vh_{i-1}
=\mW_i(\boldsymbol{\tau}_{i-1}-\vh_{i-1})$, which is equal to the IFA recursion. By induction,
\begin{equation}
    \boldsymbol{\tau}_i-\vh_i \;=\; \delta_i \;=\; \mW_i\mW_{i-1}\cdots\mW_2\,\mG\ve ,
    \qquad 1\le i<L .
    \label{eq:ftp_ifa_linear}
\end{equation}
The hidden-layer updates of the two methods therefore coincide, and the output-layer updates
are identical by construction. Since identical updates applied to identical initial weights
produce identical trajectories, FTP and IFA are equivalent at every training step in the linear
setting, i.e., $\Delta\mW_i^{\mathrm{FTP}}=\Delta\mW_i^{\mathrm{IFA}}$ for all layers.
Consequently, \cref{lemma_1} and both theorems of Appendix~\ref{appendix_theory} hold for IFA
with $\mB=\mG$, and the alignment of the update with BP and with the Gauss--Newton direction
in this setting is a property shared by the two methods.

For a general, twice-differentiable $\sigma$ (e.g., \textit{tanh}), a first-order expansion of
the FTP target differences gives
\begin{align}
    \boldsymbol{\tau}_1-\vh_1 &= \sigma(\mG\vy)-\sigma(\mG\vh_L)
      = \sigma'(\mG\vh_L)\odot\mG\ve + \mathcal{O}(\lVert\ve\rVert^2), \\
    \boldsymbol{\tau}_i-\vh_i &= \sigma(\mW_i\boldsymbol{\tau}_{i-1})-\sigma(\mW_i\vh_{i-1})
      = \sigma'(\va_i)\odot\mW_i(\boldsymbol{\tau}_{i-1}-\vh_{i-1})
      + \mathcal{O}(\lVert\boldsymbol{\tau}_{i-1}-\vh_{i-1}\rVert^2), \quad 1<i<L,
\end{align}
where the first line uses $\mG\vy=\mG\vh_L+\mG\ve$. Since
$\boldsymbol{\tau}_{i-1}-\vh_{i-1}=\mathcal{O}(\lVert\ve\rVert)$ by induction, all remainder
terms are $\mathcal{O}(\lVert\ve\rVert^2)$.

Comparing these expressions with the IFA recursion reveals three differences. First, the
signals are seeded differently at the first hidden layer. Including the derivative in the
weight update (Eq.~\ref{eqn: w_update}), the effective FTP signal for $\mW_1$ is, to first
order, $\sigma'(\va_1)\odot\sigma'(\mG\vh_L)\odot\mG\ve$, whereas the IFA signal is
$\sigma'(\va_1)\odot\mB\ve$. FTP thus carries an additional modulation by the derivative of the
nonlinearity at the projected output, so the top-down signal is shaped in the space of the
feedback projection. Moreover, the quantity propagated upward in FTP is
$\boldsymbol{\tau}_1-\vh_1$, which does not contain $\sigma'(\va_1)$, whereas IFA propagates
$\delta_1$, which does. Second, although the two recursions share the same structure for
$1<i<L$, with each step multiplying by $\mW_i$ and gating by $\sigma'(\va_i)$, FTP applies the
local derivative a second time in the weight update, so that its effective signal at layer $i$
is $\sigma'(\va_i)\odot(\boldsymbol{\tau}_i-\vh_i)$, compared with $\delta_i$ for IFA. Third, FTP
computes exact differences through the nonlinear forward map, whereas the IFA signal is linear
in $\ve$ by construction. FTP's targets respect the saturation
of $\sigma$, whereas a linearized error signal does not. The scaling parameter $\gamma$ of
Section~\ref{section_alignment} scales the first-layer target shift and therefore controls the
relative weight of the higher-order terms in the propagation to deeper layers, which may
contribute to the stronger alignment with BP observed for smaller $\gamma$.

\section{Model Architecture and Implementation Details}
\label{appendix_model_details}

We evaluated the performance of FTP, BP, DTP, and PEPITA across three model families--fully connected networks (FC), convolutional neural networks (CNN), and recurrent neural networks (RNN)--on image classification and multivariate time-series forecasting tasks. While Algorithm~\ref{alg:FTP} and Eqs.~(1)--(4) are presented using FC networks for clarity, the underlying principles of FTP apply equally to CNNs and RNNs with only minor modifications. The core learning algorithm and update equations remain unchanged across architectures; the only adjustment required is reshaping the feedback projection to match the dimensionality of the corresponding hidden activations. For example, in CNNs, a feedback matrix $G \in \mathbb{R}^{C \times (C_{1} \cdot H_{1} \cdot W_{1})}$ projects the output predictions or labels $y \in \mathbb{R}^{B \times C}$ into a correction matrix of shape $(B, C_{1} \cdot H_{1} \cdot W_{1})$. This is then reshaped into a tensor of shape $(B, C_{1}, H_{1}, W_{1})$, matching the spatial dimensions of the first hidden layer activations. Subsequent targets are propagated forward as in the first forward pass, and weight updates follow the same principles as in the FC case. 

All models were trained using stochastic gradient descent (SGD) with momentum 0.9, batch size 64, and cross-entropy loss. Unless otherwise stated, all hidden layers used the \textit{tanh} activation function, and all parameters (including the projection matrix $\mG$ in FTP) were initialized using He initialization~\citep{he2015delving}.

\subsection{Fully Connected Networks (FC)}

We evaluated FC networks on MNIST \citep{mnist}, Fashion-MNIST (FMNIST) \citep{Xiao2017FashionMNISTAN}, and CIFAR-10 \citep{cifar10}. For MNIST and FMNIST, input images were flattened to 784 dimensions; for CIFAR datasets, to 3072 dimensions. The architecture consisted of two hidden layers with 1024 and 128 neurons, followed by a softmax output layer with 10 units for MNIST/FMNIST/CIFAR-10. A dropout rate of 0.1 was applied after each hidden layer. All models were trained for 100 epochs, with the learning rate decayed by a factor of 10 at epochs 60 and 90.

\subsection{Convolutional Neural Networks (CNN)}

We applied CNNs to the same four image classification datasets: MNIST \citep{mnist}, CIFAR-10, and CIFAR-100 \citep{cifar10}. The CNN architecture consisted of a 2D convolutional layer with 32 output channels and a 5$\times$5 kernel, followed by a 2$\times$2 max pooling layer. The output feature maps were flattened and passed to a softmax output layer with 10 or 100 units, depending on the dataset. All CNN models used \textit{tanh} activations except the final layer where softmax activation function was used. The models were trained for 100 epochs, and followed the same learning rate schedule as the FC models (decayed at epochs 60 and 90).

\subsection{Recurrent Neural Networks (RNN)}
\label{app:ftp_rnn}

For a single recurrent layer with input weights $\mW_{xh}$, recurrent weights $\mW_{hh}$, and readout $\mW_{hy}$, the first forward pass over an input window $\vx_1, \dots, \vx_T$ is:
\begin{equation}
\vh_t = \sigma_h(\mW_{xh}\vx_t + \mW_{hh}\vh_{t-1} + \vb_h), \qquad
\hat{\vy} = \mW_{hy}\vh_T + \vb_y,
\end{equation}
with $\vh_0 = \mathbf{0}$ and $\sigma_h$ the activation function.

FTP forms the target sequence by re-running the same recurrence with a single modification: the initial hidden state is set to the difference-corrected projection of the output target and prediction through a fixed random matrix $\mG \in \mathbb{R}^{d_h \times d_{\text{out}}}$,
\begin{equation}
\boldsymbol{\tau}_0 = \sigma(\mG\vy) - \sigma(\mG\hat{\vy}),
\label{eq:rnn_init_target}
\end{equation}
after which targets propagate forward through the same recurrent weights used in the first pass:
\begin{equation}
\boldsymbol{\tau}_t = \sigma_h(\mW_{xh}\vx_t + \mW_{hh}\boldsymbol{\tau}_{t-1} + \vb_h), \qquad t = 1,\dots,T .
\label{eq:rnn_target_prop}
\end{equation}
The two passes are therefore identical apart from their initial condition: the first begins from $\vh_0 = \mathbf{0}$, the second from $\boldsymbol{\tau}_0$. The error enters the computation exactly once and is carried forward by the recurrence, in direct analogy to $\boldsymbol{\tau}_i = \sigma(\mW_i \boldsymbol{\tau}_{i-1})$ in the feedforward case (Section~\ref{learning_method}). Because $\mW_{xh}$ and $\mW_{hh}$ are shared across timesteps, the unrolled network is a weight-tied layer stack, and target propagation through it is the same recursion applied to that stack. No error is injected at any subsequent timestep, and no feedback pathway is required beyond the single projection $\mG$.

Each timestep is trained by a local loss between its activation and its target, and the readout by the task loss:
\begin{equation}
\mathcal{L}_t = \tfrac{1}{2}\big\lVert \vh_t - \texttt{stop\_grad}(\boldsymbol{\tau}_t) \big\rVert_2^2,
\qquad
\mathcal{L}_{\text{out}} = \tfrac{1}{2}\big\lVert \hat{\vy} - \vy \big\rVert_2^2 .
\end{equation}
The shared weights accumulate the local gradients across the unrolled sequence, with each timestep's contribution computed solely from its own local loss and $\vh_{t-1}$ held fixed:
\begin{align}
\Delta \mW_{xh} &= -\eta \sum_{t=1}^{T} \big(\vh_t - \boldsymbol{\tau}_t\big) \odot \sigma_h'(\vz_t) \; \vx_t^{\top}, \\
\Delta \mW_{hh} &= -\eta \sum_{t=1}^{T} \big(\vh_t - \boldsymbol{\tau}_t\big) \odot \sigma_h'(\vz_t) \; \vh_{t-1}^{\top}, \\
\Delta \mW_{hy} &= -\eta \, \big(\hat{\vy} - \vy\big) \, \vh_T^{\top},
\end{align}
where $\vz_t = \mW_{xh}\vx_t + \mW_{hh}\vh_{t-1} + \vb_h$ and $\sigma_h'(\vz_t)$ is derivative of the activation function. Updates are applied with momentum $0.9$, as in the feedforward experiments.

Holding $\vh_{t-1}$ fixed when differentiating $\mathcal{L}_t$ is what distinguishes FTP-RNN from BPTT: the computational graph is never traversed in reverse temporal order, no error is transported through transposed weight matrices, and all computation proceeds forward in time. FTP-RNN does, however, retain the storage of per-timestep activations, since $\vh_t$ must be held until $\tau_t$ is computed. This corresponds to a synaptic eligibility trace rather than a stored computational graph, and we discuss it as a limitation in Section~\ref{sec:limitations}.

\begin{algorithm}[H]
\caption{FTP for Recurrent Networks}
\label{alg:FTP_RNN}
\begin{algorithmic}
   \STATE {\bfseries Given:} Input window $\vx_1,\dots,\vx_T$, output target $\vy$
   \STATE $\vh_0 = \mathbf{0}$
   \STATE \textcolor{gray}{\# First Forward Pass}
   \FOR{$t = 1$ {\bfseries to} $T$}
       \STATE $\vh_t = \sigma_h(\mW_{xh}\vx_t + \mW_{hh}\vh_{t-1} + \vb_h)$
   \ENDFOR
   \STATE $\hat{\vy} = \mW_{hy}\vh_T + \vb_y$
   \STATE \textcolor{gray}{\# Initialize Target State}
   \STATE $\boldsymbol{\tau}_0 = \texttt{stop\_grad}\big(\sigma(\mG\vy) - \sigma(\mG\hat{\vy})\big)$
   \STATE \textcolor{gray}{\# Second Forward Pass (Targets)}
   \FOR{$t = 1$ {\bfseries to} $T$}
       \STATE $\tau_t = \texttt{stop\_grad}\big(\sigma_h(\mW_{xh}\vx_t + \mW_{hh}\boldsymbol{\tau}_{t-1} + \vb_h)\big)$
   \ENDFOR
   \STATE \textcolor{gray}{\# Local Updates}
   \FOR{$t = 1$ {\bfseries to} $T$}
       \STATE $\Delta \mW_{xh} \mathrel{-}= \eta \nabla_{\mW_{xh}} \mathcal{L}_t$, \quad
              $\Delta \mW_{hh} \mathrel{-}= \eta \nabla_{\mW_{hh}} \mathcal{L}_t$
   \ENDFOR
   \STATE $\Delta \mW_{hy} = -\eta \nabla_{\mW_{hy}} \mathcal{L}_{\text{out}}$
\end{algorithmic}
\end{algorithm}

We evaluated RNNs on three multivariate time-series datasets:

\begin{itemize}
    \item \textit{Electricity}: Electricity consumption data from 321 clients, recorded every 15 minutes from 2012 to 2014, and resampled to hourly resolution.
    \item \textit{METR-LA}: Hourly road occupancy rates from 2015--2016, recorded by sensors on Los Angeles-area loop-detector data.
    \item \textit{Solar-Energy}: Solar power output sampled every 10 minutes during 2006 from 137 photovoltaic plants in Alabama.

\end{itemize}

Each task was structured as a sequence prediction problem, where a sliding window of the
previous $T$ time steps was used to predict the next. We evaluate
$T \in \{24, 48, 96, 120\}$; results for $T=24$ appear in Table~\ref{tab:rnn_results},
and the remaining input windows in Appendix~\ref{appendix_horizons}. The RNN architecture
consisted of a single recurrent layer with 512 hidden units. All RNN models were trained
for 500 epochs, with the learning rate decayed by a factor of 10 at epochs 300 and 450.

\subsection{Autoencoder}
\label{app:ae}

FTP is target-driven rather than label-driven. What the algorithm requires is a signal at the output layer, which the fixed matrix $\mG$ projects into the hidden layers; nothing in the derivation of Eqs. (1)–(6) assumes that this signal is a one-hot class vector. FTP is designed as a supervised learning rule, but it works for all those cases which may not have class labels as output. Image Reconstruction is the natural such case: the target is the input itself.

To demonstrate this concretely,  we trained an autoencoder with the  architecture:
$d \rightarrow 512 \rightarrow 128 \rightarrow 512 \rightarrow d$ (where $d$ is the
input dimension of each dataset), with the
128-dimensional bottleneck serving as the learned representation. All layers are
bias-free with \texttt{tanh} hidden activations and a sigmoid reconstruction layer,
initialized with He-uniform initialization as in the FC experiments. A dropout rate of
0.1 is applied after each hidden layer, with the same mask presented in both forward
passes.

The only change relative to the supervised formulation is the target: \( \vy = \mX \),
the input image itself. The projection \( \mG \in \mathbb{R}^{512 \times d} \) maps
the pixel-space target into the first hidden layer, drawn from a He-uniform
distribution. All local update rules are unchanged.

The output layer is trained with binary cross-entropy on a sigmoid output, so that the
derivative with respect to the output pre-activation is exactly \( \hat{\vy} - \vy \)
and the final-layer update retains the same form as in the softmax--cross-entropy case.
Models were trained for 100 epochs using SGD with momentum 0.9, learning rate 0.01, and
batch size 128, with the learning rate decayed by a factor of 10 at epochs 60 and 90.
The BP baseline is identical except that gradients are backpropagated through the
network.
For the linear probe, the encoder is frozen and the 128-dimensional bottleneck is
extracted with dropout disabled. A single linear layer
\( \mathbb{R}^{128} \rightarrow \mathbb{R}^{C} \) (where $C$ is the number of classes)
is trained on the frozen features
with cross-entropy loss, using SGD with momentum of 0.9 and learning rate $10^{-2}$, batch size 128, for
100 epochs. Results are averaged over 10 seeds.

Performance is assessed along two axes. Reconstruction loss measures the fidelity of
the learned encoding--decoding map, while a linear probe, a single linear
classifier trained on the frozen 128-dimensional bottleneck, quantifies how much
class-relevant structure the network recovers without supervision. Table~\ref{tab:autoencoder}
reports both across MNIST, Fashion-MNIST, CIFAR-10, and CIFAR-100.
Across all four datasets, FTP tracks BP closely, with linear-probe accuracy within
roughly 0.5--1.7\% and comparable reconstruction error.

\begin{table}[h]
\centering
\caption{Autoencoder trained on datasets without labels. Linear probe is trained on the
frozen 128- dimensional bottleneck. Mean $\pm$ std over 10 seeds.}
\label{tab:autoencoder}
\begin{tabular}{llcc}
\toprule
Dataset & Method & Reconstruction loss (BCE) $\downarrow$ & Linear-probe accuracy (\%) $\uparrow$ \\
\midrule
\multirow{2}{*}{MNIST}           & BP  & 0.0924 $\pm$ 0.007  & 91.51 $\pm$ 0.09 \\
                                 & FTP & 0.0901 $\pm$ 0.010  & 91.28 $\pm$ 0.16 \\
\midrule
\multirow{2}{*}{Fashion-MNIST}   & BP  & 0.2802 $\pm$ 0.0002 & 83.90 $\pm$ 0.12 \\
                                 & FTP & 0.2859 $\pm$ 0.0008 & 83.40 $\pm$ 0.07 \\
\midrule
\multirow{2}{*}{CIFAR-10}        & BP  & 0.5816 $\pm$ 0.0000 & 39.30 $\pm$ 0.13 \\
                                 & FTP & 0.5963 $\pm$ 0.0025 & 38.55 $\pm$ 0.46 \\
\midrule
\multirow{2}{*}{CIFAR-100}       & BP  & 0.5606 $\pm$ 0.0000 & 17.92 $\pm$ 0.16 \\
                                 & FTP & 0.5759 $\pm$ 0.0008 & 17.40 $\pm$ 0.52 \\
\bottomrule
\end{tabular}
\end{table}

\section{Sensitivity to Initialization and Dimensionality of Feedback Matrix \texorpdfstring{$\mG$}{G}}
\label{sensitivity_G_appendix}

The feedback matrix $\mG$ in FTP plays a crucial role in projecting the local error signal (Eq.~\ref{first_eqn}). In our experiments, we initialized $\mG$ using He Uniform~\citep{he2015delving}, scaled by a constant factor (0.05) to regulate the strength of the feedback signal. Since $\mG$ directly affects the magnitude of projected error, its initialization impacts learning dynamics and gradient alignment.

To assess sensitivity to initialization schemes, we additionally trained the FC model with \textit{He Normal} initialization~\citep{he2015delving}, using a similar scaling factor. As shown in Table~\ref{tab:init_sensitivity}, FTP yields comparable performance under both schemes, suggesting that while initialization matters, the method is not overly sensitive to the exact distribution as long as the scale is appropriately controlled.

\begin{table}[h]
\centering
\caption{FTP accuracy (\%) under different initialization schemes for $\mG$. Mean {\scriptsize$\pm$ std} over 5 trials.}
\label{tab:init_sensitivity}
\begin{tabular}{lccc}
\toprule
\textbf{Initialization} & \textbf{MNIST} & \textbf{FMNIST} & \textbf{CIFAR-10} \\
\midrule
He Uniform & 97.98{\scriptsize$\pm$0.25} & 87.24{\scriptsize$\pm$0.21} & 52.57{\scriptsize$\pm$0.37} \\
He Normal  & 97.51{\scriptsize$\pm$0.31} & 87.33{\scriptsize$\pm$0.18} & 53.44{\scriptsize$\pm$0.26} \\
\bottomrule
\end{tabular}
\end{table}

The dimensionality of $\mG$ depends on the number of output classes and the hidden layer size. As shown in Table~\ref{test_acc}, FTP achieves strong performance on both CIFAR-10 and CIFAR-100, which differ significantly in class count (10 vs. 100), indicating generalizability across different output dimensionalities.

To further isolate the effect of class count, we reformulated MNIST into 5-class and 2-class variants (e.g., grouping digits by adjacent pairs or by even/odd labels). As shown in Table~\ref{tab:class_sensitivity}, FTP maintains performance comparable to BP across all class settings. These results indicate that FTP is robust to variations in $\mG$'s dimensionality induced by class count.

\begin{table}[h]
\centering
\caption{Accuracy (\%) on MNIST variants with reduced class count. Mean {\scriptsize$\pm$ std} over 5 trials.}
\label{tab:class_sensitivity}
\begin{tabular}{lccc}
\toprule
\textbf{Method} & \textbf{MNIST (10)} & \textbf{MNIST (5)} & \textbf{MNIST (2)} \\
\midrule
BP  & 98.27{\scriptsize$\pm$0.08} & 98.14{\scriptsize$\pm$0.06} & 98.95{\scriptsize$\pm$0.04} \\
FTP & 97.98{\scriptsize$\pm$0.25} & 97.81{\scriptsize$\pm$0.16} & 98.66{\scriptsize$\pm$0.12} \\
\bottomrule
\end{tabular}
\end{table}

\section{Metrics for Evaluation of RNNs}
\label{appendix_rrse_corrr}
Recurrent neural networks are commonly evaluated not only by accuracy but also by statistical measures that capture both error magnitude and temporal correlation. In this work, we use Root Relative Squared Error (RRSE) and the Pearson correlation coefficient (CORR), defined as follows:

\begin{equation}
\mathrm{RRSE} = \sqrt{ \frac{\sum_{t} (y_t - \hat{y}_t)^2}{\sum_{t} (y_t - \bar{y})^2} }, \quad
\mathrm{CORR} = \frac{\sum_{t} (y_t - \bar{y})(\hat{y}_t - \bar{\hat{y}})}{\sqrt{\sum_{t} (y_t - \bar{y})^2} \sqrt{\sum_{t} (\hat{y}_t - \bar{\hat{y}})^2}}
\end{equation}
\vskip -3mm

where $y_t$ is the ground truth value, $\hat{y}_t$ is the predicted value, $\bar{y}$ and $\bar{\hat{y}}$ are the means of the ground truth and predictions, respectively.

RRSE (Root Relative Squared Error) normalizes the prediction error relative to the variance of the true series, which ensures comparability across datasets of different scales. A lower RRSE value indicates better predictive performance. CORR (Pearson correlation coefficient) measures the linear relationship between the predicted and actual sequences, capturing how well the model tracks temporal patterns. Values of CORR closer to 1 imply stronger correlation and more faithful temporal modeling. Together, these two complementary metrics provide a robust evaluation: RRSE quantifies absolute predictive accuracy, while CORR highlights alignment with the temporal dynamics of the data.


\section{Depth Scaling of Forward Learning Algorithms}
\label{scalability_appendix}

Our initial experiments in Section \ref{original_results} used a two-layer FC network with hidden dimensions of 1024–128, but under this setting PEPITA exhibited unstable training for deeper networks. To ensure a fair comparison, and following the setup in \citet{pepita_srini}, we standardized the hidden layer size to 1024 units across all layers, allowing us to scale architectures to 2–5 hidden layers while maintaining stability for both FTP and PEPITA. Consistent with prior observations by \citet{pepita_srini}, PEPITA struggled to scale beyond three hidden layers: normalization improved convergence for deeper models but accuracy still declined with depth. In our experiments without normalization (Table~\ref{tab:bp-pep-ftp}), PEPITA’s performance degraded as depth increased, whereas FTP remained stable and achieved accuracy close to BP without requiring normalization or additional techniques such as weight decay or weight mirroring. These results suggest that FTP scales more reliably to deeper FC networks without auxiliary mechanisms, whereas PEPITA shows reduced stability beyond shallow models. On the other hand, Forward-Forward Learning has been limited to four hidden layers \citep{Hinton2022TheFA}. This points to a significant improvement over prior forward-only learning approaches.

For convolutional networks, we extended the baseline 1-Conv design by adding a second convolutional layer (3×3, 64 channels). In these settings, FTP again outperforms PEPITA, demonstrating its potential to scale effectively to CNNs.
\begin{table}[htb!]
\centering
\caption{Test accuracy (Mean {\scriptsize$\pm$ Std}) of BP, PEPITA, and FTP on CNNs with 1 and 2 convolutional layers.}
\vspace{-2mm}
\label{tab:cnn-bp-pep-ftp}
\small
\renewcommand{\arraystretch}{1.2}
\begin{tabular}{l l c c}
\toprule
Dataset & Algorithm & 1-Conv & 2-Conv \\
\midrule
\multirow{3}{*}{MNIST} 
 & BP      & 98.74{\scriptsize$\pm$0.05} & 99.05{\scriptsize$\pm$0.03} \\
 & PEPITA  & 98.41{\scriptsize$\pm$0.24} & 98.35{\scriptsize$\pm$0.05} \\
 & FTP & 98.28{\scriptsize$\pm$0.35} & 98.69{\scriptsize$\pm$0.07} \\
\midrule
\multirow{3}{*}{CIFAR-10} 
 & BP      & 64.88{\scriptsize$\pm$0.18} & 69.42{\scriptsize$\pm$0.13} \\
 & PEPITA  & 56.17{\scriptsize$\pm$0.62} & 57.01{\scriptsize$\pm$0.58} \\
 & FTP & 56.32{\scriptsize$\pm$0.82} & 59.57{\scriptsize$\pm$0.13} \\
\midrule
\multirow{3}{*}{CIFAR-100} 
 & BP      & 33.83{\scriptsize$\pm$0.26} & 37.75{\scriptsize$\pm$0.21} \\
 & PEPITA  & 26.77{\scriptsize$\pm$0.87} & 27.61{\scriptsize$\pm$0.20} \\
 & FTP & 26.84{\scriptsize$\pm$1.13} & 29.95{\scriptsize$\pm$0.22} \\
\bottomrule
\end{tabular}
\end{table}

\section{Forecasting at Longer Input Windows for Sequential Tasks}
\label{appendix_horizons}

The sequential experiments in Section~\ref{section_rnn} use a 24-timestep input
window. To test whether FTP remains stable as the unrolled recurrent depth grows,
we repeat the evaluation on all three benchmarks: Electricity, METR-LA (Traffic),
and Solar Energy \citep{Lai2017ModelingLA} at input windows of $\{24, 48, 96,
120\}$ timesteps. The architecture, initialization, optimizer, and training
schedule are identical to those in Appendix~\ref{app:ftp_rnn} and
Section~\ref{section_rnn}; only the window length changes. We report Root Relative
Squared Error (RRSE) and the Pearson correlation coefficient (CORR), defined in
Appendix~\ref{appendix_rrse_corrr}, as mean $\pm$ std over 10 random seeds, and
compare FTP-RNN against BP (BPTT) and PEPITA-RNN. The 24-timestep rows reproduce
Table~\ref{tab:rnn_results}.

\begin{table}[h]
\centering
\caption{Forecasting performance across input windows $\{24, 48, 96, 120\}$ on the
Electricity, METR-LA, and Solar benchmarks (mean $\pm$ std over 10 seeds).
$\uparrow$ ($\downarrow$) indicates that higher (lower) values are better.}
\label{tab:rnn_windows}
\small
\setlength{\tabcolsep}{4pt}
\begin{tabular}{llcccccc}
\toprule
\multirow{2}{*}{Win.} & \multirow{2}{*}{Method}
 & \multicolumn{2}{c}{Electricity} & \multicolumn{2}{c}{METR-LA} & \multicolumn{2}{c}{Solar} \\
\cmidrule(lr){3-4} \cmidrule(lr){5-6} \cmidrule(lr){7-8}
 & & RRSE$\downarrow$ & CORR$\uparrow$ & RRSE$\downarrow$ & CORR$\uparrow$ & RRSE$\downarrow$ & CORR$\uparrow$ \\
\midrule
\multirow{3}{*}{24}
 & BP (BPTT)        & \pmstd{0.1059}{0.0013} & \pmstd{0.9302}{0.0011} & \pmstd{0.4680}{0.0083} & \pmstd{0.8750}{0.0014} & \pmstd{0.1161}{0.0009} & \pmstd{0.9919}{0.0009} \\
 & PEPITA-RNN       & \pmstd{0.1214}{0.0019} & \pmstd{0.9929}{0.0009} & \pmstd{0.4419}{0.0065} & \pmstd{0.8967}{0.0022} & \pmstd{0.1170}{0.0017} & \pmstd{0.9932}{0.0022} \\
 & \textbf{FTP-RNN} & \pmstd{0.1219}{0.0010} & \pmstd{0.9935}{0.0010} & \pmstd{0.4398}{0.0073} & \pmstd{0.8987}{0.0019} & \pmstd{0.1177}{0.0011} & \pmstd{0.9931}{0.0019} \\
\midrule
\multirow{3}{*}{48}
 & BP (BPTT)        & \pmstd{0.1203}{0.0005} & \pmstd{0.9933}{0.0010} & \pmstd{0.4413}{0.0012} & \pmstd{0.8975}{0.0007} & \pmstd{0.1090}{0.0006} & \pmstd{0.9940}{0.0010} \\
 & PEPITA-RNN       & \pmstd{0.1682}{0.0013} & \pmstd{0.9904}{0.0018} & \pmstd{0.4607}{0.0018} & \pmstd{0.8878}{0.0011} & \pmstd{0.1394}{0.0015} & \pmstd{0.9902}{0.0012} \\
 & \textbf{FTP-RNN} & \pmstd{0.1282}{0.0010} & \pmstd{0.9954}{0.0008} & \pmstd{0.4521}{0.0015} & \pmstd{0.8948}{0.0010} & \pmstd{0.1144}{0.0005} & \pmstd{0.9917}{0.0017} \\
\midrule
\multirow{3}{*}{96}
 & BP (BPTT)        & \pmstd{0.1230}{0.0008} & \pmstd{0.9930}{0.0006} & \pmstd{0.4421}{0.0006} & \pmstd{0.8971}{0.0008} & \pmstd{0.1089}{0.0004} & \pmstd{0.9941}{0.0012} \\
 & PEPITA-RNN       & \pmstd{0.1699}{0.0011} & \pmstd{0.9902}{0.0007} & \pmstd{0.4601}{0.0019} & \pmstd{0.8880}{0.0020} & \pmstd{0.1390}{0.0009} & \pmstd{0.9903}{0.0012} \\
 & \textbf{FTP-RNN} & \pmstd{0.1399}{0.0013} & \pmstd{0.9919}{0.0014} & \pmstd{0.4372}{0.0013} & \pmstd{0.8939}{0.0020} & \pmstd{0.1139}{0.0009} & \pmstd{0.9914}{0.0012} \\
\midrule
\multirow{3}{*}{120}
 & BP (BPTT)        & \pmstd{0.1214}{0.0007} & \pmstd{0.9932}{0.0005} & \pmstd{0.4413}{0.0008} & \pmstd{0.8975}{0.0004} & \pmstd{0.1090}{0.0009} & \pmstd{0.9940}{0.0014} \\
 & PEPITA-RNN       & \pmstd{0.1724}{0.0017} & \pmstd{0.9903}{0.0011} & \pmstd{0.4621}{0.0015} & \pmstd{0.8870}{0.0013} & \pmstd{0.1395}{0.0006} & \pmstd{0.9902}{0.0014} \\
 & \textbf{FTP-RNN} & \pmstd{0.1387}{0.0010} & \pmstd{0.9942}{0.0004} & \pmstd{0.4557}{0.0016} & \pmstd{0.8954}{0.0013} & \pmstd{0.1170}{0.0008} & \pmstd{0.9915}{0.0010} \\
\bottomrule
\end{tabular}
\end{table}

FTP remains stable as the unrolled depth grows: across a fivefold increase in
window length (24 to 120), RRSE stays within a narrow band on every dataset and
CORR remains above $0.89$ on METR-LA and above $0.99$ on Electricity and Solar.
FTP-RNN outperforms PEPITA-RNN on both metrics, on every dataset, at the 48, 96, and 120-step windows, with non-overlapping RRSE error bands; the two are comparable at the 24-step window, so FTP-RNN's advantage over PEPITA-RNN grows as the window length increases. It also matches or exceeds BP in correlation at
several settings, achieving the highest CORR of any method on Electricity at the
48 and 120-step windows ($0.9954$ and $0.9942$, versus BP's $0.9933$ and
$0.9932$), and the lowest RRSE of any method on METR-LA at the 96-step window
($0.4372$ versus BP's $0.4421$), in each case with separated error bands. BP
retains a modest RRSE advantage on Electricity and Solar. Overall, lengthening the
input window does not degrade FTP relative to BPTT, indicating that its
forward-only target propagation assigns credit reliably across longer temporal
unrolls.

\section{Robustness of FTP compared to Feedback Alignment}
\label{robustness_appendix}

To evaluate the robustness of FTP against Feedback Alignment (FA)\citep{lillicrap2016random}, we conducted experiments under the same setup as described in \cref{sec:robust_FTP}, across 10 independent trials. Unlike backpropagation, FA employs fixed random feedback matrices, eliminating the need to reprogram backward paths in hardware. Although FA is less biologically plausible than FTP, its fixed feedback makes it an important baseline for robustness under hardware constraints. We tested both 8-bit and 4-bit weight precision settings under varying levels of additive noise (0, 0.02, 0.06, 0.1). The results are reported in Table \ref{tab:robustness}.




\begin{table}[h]
\centering
\caption{Robustness comparison of BP, FA, and FTP under 8-bit and 4-bit precision with increasing noise levels. Accuracy reported as mean {\scriptsize$\pm$ std} over 10 trials.}
\label{tab:robustness}
\vspace{0.5em}

\begin{tabular}{lcccc}
\toprule
 & \multicolumn{4}{c}{\textbf{Noise Levels}} \\
\cmidrule(lr){2-5}
\textbf{Method (8-bit)} & \textbf{0} & \textbf{0.02} & \textbf{0.06} & \textbf{0.1} \\
\midrule
BP (symmetry = 100\%) & \textbf{97.91}{\scriptsize$\pm$0.08} & \textbf{95.60}{\scriptsize$\pm$0.07} & \textbf{90.81}{\scriptsize$\pm$0.17} & 84.56{\scriptsize$\pm$0.81} \\
FA                    & 97.50{\scriptsize$\pm$0.08} & 94.29{\scriptsize$\pm$0.19} & 88.10{\scriptsize$\pm$0.53} & 81.54{\scriptsize$\pm$1.26} \\
FTP                   & 97.49{\scriptsize$\pm$0.19} & 95.17{\scriptsize$\pm$0.22} & 90.48{\scriptsize$\pm$0.29} & \textbf{85.85}{\scriptsize$\pm$0.56} \\
\bottomrule
\end{tabular}

\vspace{1.0em}

\begin{tabular}{lcccc}
\toprule
 & \multicolumn{4}{c}{\textbf{Noise Levels}} \\
\cmidrule(lr){2-5}
\textbf{Method (4-bit)} & \textbf{0} & \textbf{0.02} & \textbf{0.06} & \textbf{0.1} \\
\midrule
BP (symmetry = 100\%) & \textbf{94.48}{\scriptsize$\pm$0.11} & 90.12{\scriptsize$\pm$0.15} & 81.25{\scriptsize$\pm$1.34} & 11.35{\scriptsize$\pm$0.00} \\
FA                    & 93.69{\scriptsize$\pm$0.34} & 90.24{\scriptsize$\pm$0.38} & 78.34{\scriptsize$\pm$6.71} & 17.83{\scriptsize$\pm$10.54} \\
FTP                   & 93.52{\scriptsize$\pm$0.25} & \textbf{90.89}{\scriptsize$\pm$0.27} & \textbf{82.48}{\scriptsize$\pm$2.53} & \textbf{24.73}{\scriptsize$\pm$16.98} \\
\bottomrule
\end{tabular}
\end{table}

As reflected in Figure \ref{fig:programming_error} and the Table \ref{tab:robustness}, we observe that under ideal conditions (i.e., no noise), BP and FA achieve slightly higher accuracy than FTP in low bit-precision settings. However, as noise increases, FTP shows consistent and greater robustness. This suggests that FTP has inherent robustness properties that enable it to maintain stable learning dynamics under realistic hardware constraints such as low-bit precision and noise. 

Importantly, since FTP uses the same forward matrices to compute both the standard activations and the layer-wise targets, any noise introduced in these matrices affects both pathways in a consistent manner. In contrast, BP requires separate programming of the forward and backward matrices in each iteration, which means programming noise can independently corrupt the two directions. This also applies for FA. This difference may have a greater impact on BP’s updates compared to FTP, which may explain FTP’s more stable performance in these noisy settings. In real analog hardware, additional non-idealities, such as device drift, retention loss, and other errors from the system could further exacerbate this asymmetry for BP, especially due to its reliance on reprogramming of backward matrices. Investigating these broader effects through real hardware implementation remains an important direction for future work. 

\section{Comparison with Additional Forward-Learning Baselines}
\label{appendix_forward_baselines}

Here, we compare  FTP against a broader set of forward-learning methods along with PEPITA: Fixed-Weight Difference Target Propagation (FWDTP)~\citep{shibuya2021fixed},
which replaces DTP's learned inverses with fixed random feedback, and Signal Propagation
(SigProp)~\citep{kohan2022SigProp}, which delivers a global learning signal through the
feedforward pathway. We evaluate both under the two lenses already used in the main text:
the multiply--accumulate (MAC) cost analysis of Section~\ref{sectionfc&cnn} and the
low-precision programming-noise protocol of Section~\ref{sec:robust_FTP}
(and Appendix~\ref{robustness_appendix}).

\paragraph{Computational cost.}
Table~\ref{tab:mac_forward_baselines} reports the per-sample MAC count of each method as a
percentage increase over BP on MNIST, CIFAR-10, and CIFAR-100. FWDTP and SigProp both
require $32$--$53\%$ more MACs than BP, in the same range as PEPITA and DTP, whereas FTP
stays within $1$--$3\%$ of BP.

\begin{table}[htb!]
\centering
\caption{Per-sample MAC cost of forward-learning methods, reported as percentage increase
over BP ($\downarrow$ is better).}
\label{tab:mac_forward_baselines}
\vspace{-1em}
\small
\renewcommand{\arraystretch}{1.2}
\begin{tabular}{lccc}
\toprule
Method & MNIST & CIFAR-10 & CIFAR-100 \\
\midrule
FWDTP  & +53.3\% & +33.8\% & +33.8\% \\
SigProp & +41.1\% & +32.1\% & +33.3\% \\
PEPITA  & +40.5\% & +47.4\% & +51.5\% \\
\textbf{FTP} & \textbf{+1.0\%} & \textbf{+0.3\%} & \textbf{+3.1\%} \\
\bottomrule
\end{tabular}
\end{table}

\paragraph{Robustness under quantization and programming noise.}
We repeat the programming-noise protocol of Section~\ref{sec:robust_FTP} on MNIST, injecting
additive Gaussian noise of level $\alpha$ (with standard deviation $\text{std}=\alpha\times|w|$)
at both 8-bit and 4-bit precision, and compare FTP against BP (here evaluated at $0\%$ weight
symmetry, as discussed in Appendix~\ref{appendix_weightsymmtery_bp}), PEPITA, FWDTP, and
SigProp. Figure~\ref{fig:forward_baselines_noise} shows the resulting accuracy--noise curves.

\begin{figure*}[ht!]
\vspace{0mm}
\centering
\begin{subfigure}{0.47\textwidth}
    \vskip -.1in
    \centering
    \includegraphics[width=1\linewidth]{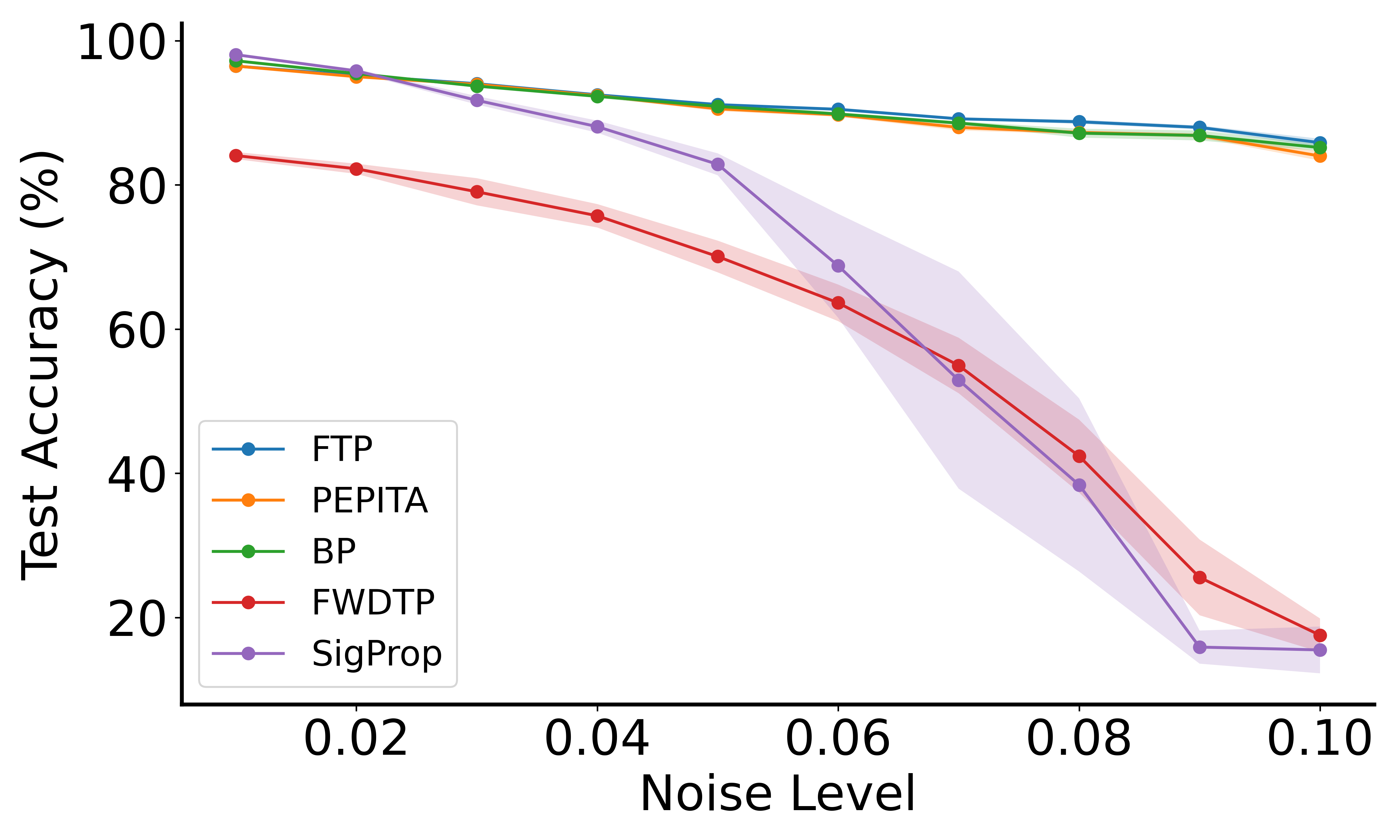}  
    \caption{} 
    \label{fig:forward_baselines_noisea}
\end{subfigure}
\hfill
\begin{subfigure}{0.47\textwidth}
    \centering
    \includegraphics[width=1\linewidth]{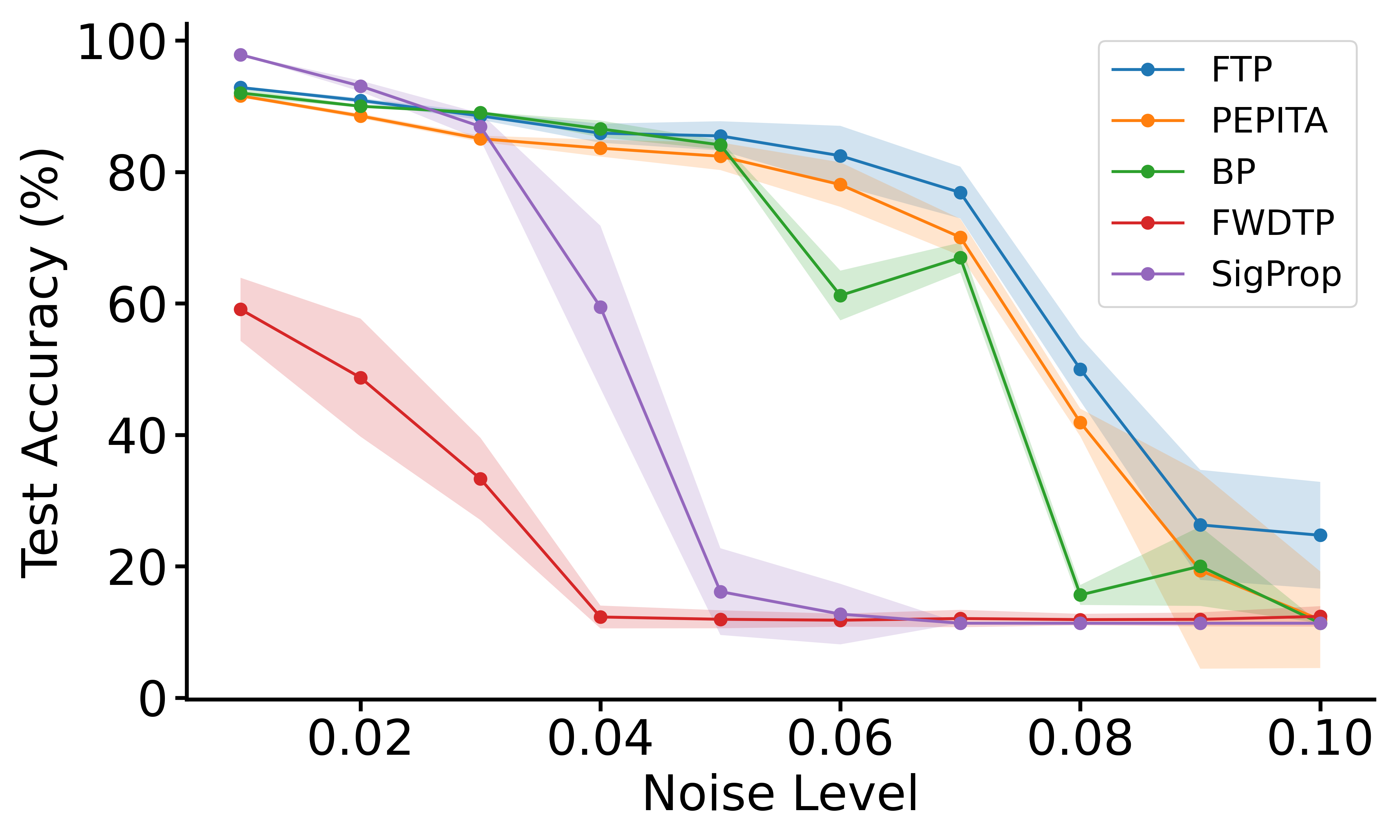}  
    \caption{} 
    \label{fig:forward_baselines_noiseb}
\end{subfigure}
\vskip -3mm
\caption{Performance of FTP and other algorithms when programming errors are considered for (a) 8-bit and (b) 4-bit precision devices.}
\label{fig:forward_baselines_noise}
\vspace{-2mm}
\end{figure*}

At 8-bit precision (Figure~\ref{fig:forward_baselines_noise}a), FTP, PEPITA, and BP remain
tightly clustered and degrade only gently across the full noise range, all retaining roughly
$85\%$ accuracy at $\alpha=0.10$, with FTP marginally highest under the strongest noise.
The two added baselines fare worse: SigProp starts highest but collapses toward chance once
$\alpha$ exceeds about $0.05$, and FWDTP declines steadily from the outset, falling below all
other methods for most of the range.
 
At 4-bit precision (Figure~\ref{fig:forward_baselines_noise}b) the separation sharpens.
FWDTP fails earliest, reaching near-chance accuracy by $\alpha\approx0.04$, and SigProp
collapses shortly after around $\alpha\approx0.05$. PEPITA and BP track FTP until
$\alpha\approx0.07$ and then drop steeply, with BP reaching chance accuracy by
$\alpha=0.10$. FTP degrades the most gracefully of all methods, retaining the highest
accuracy as noise grows and remaining well above chance where BP has already collapsed.
Overall, FTP is the only method that is simultaneously within a few percent of BP's MAC cost
and more robust than BP under low-precision programming noise.

\section{Choice of Target Projection: $\sigma(Gy)-\sigma(Gh_L)$ vs. $\sigma(G(y-h_L))$}
\label{appendix_projection_choice}

A natural question is why FTP constructs the first-layer target from the difference of two
separately projected quantities, $\sigma(Gy)-\sigma(Gh_L)$, rather than from a single
projection of their difference, $\sigma(G(y-h_L))$. The theoretical derivation in
Appendix~\ref{appendix_theory} is unchanged between the two forms, and on inspection they turn
out to be closely related.

At the projection scale used throughout the paper, the two forms are near-equivalent. Because $G$ is He-uniform, the projected pre-activations remain small, so $\sigma$ is effectively linear and the two coincide to first order, and the derivation in Appendix~\ref{appendix_theory} applies to either since the derivation considers linear activation. So,

\begin{equation}
\sigma(Gy)-\sigma(Gh_L)\;\approx\;Gy-Gh_L\;=\;Ge = G(y-h_L) \;\approx\; \sigma(G(y-h_L))
\end{equation}

Our choice follows difference target
propagation~\citep{lee2015difference}, where the target for layer $i$ is constructed as
$\tau_i = h_i + g_i(\tau_{i+1}) - g_i(h_{i+1})$, i.e., the difference between two separately
transformed quantities rather than a single transformation of their difference. FTP's
first-hidden-layer target $\tau_1 = h_1 + \sigma(Gy) - \sigma(Gh_L)$ is the direct analogue,
with a fixed random projection in place of a learned inverse. Our initial experiments also
found this form marginally more robust under low-precision perturbation, and we retained it on
that basis.
We compared the two forms systematically under the
programming-noise protocol of Section~\ref{sec:robust_FTP}, at 8-bit and 4-bit precision, over
$10$ seeds (mean and standard deviation shown as shaded regions in
Figure~\ref{fig:projection_choice}). At 8-bit (Figure~\ref{fig:projection_choice}a) the two are
effectively indistinguishable, differing by under one point at nearly every noise level with
overlapping error bands. At 4-bit (Figure~\ref{fig:projection_choice}b) the difference form
leads at most noise levels -- by roughly $2.4$ points on average for $\alpha \leq 0.08$ -- and
exhibits markedly lower variance across seeds, though the error bands overlap throughout and
the ordering is not strictly uniform. We therefore present this as the reason we retained the
difference form, not as evidence that it is necessary: consistent with the near-equivalence
above, the two perform comparably.

\begin{figure*}[htb!]
\centering
\begin{subfigure}{0.48\textwidth}
    \centering
    \includegraphics[width=\linewidth]{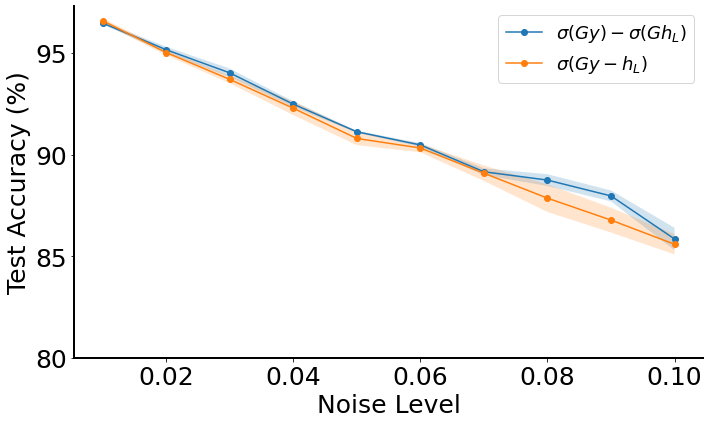}
    \caption{}
    \label{fig:projection_choice_8bit}
\end{subfigure}
\hfill
\begin{subfigure}{0.48\textwidth}
    \centering
    \includegraphics[width=\linewidth]{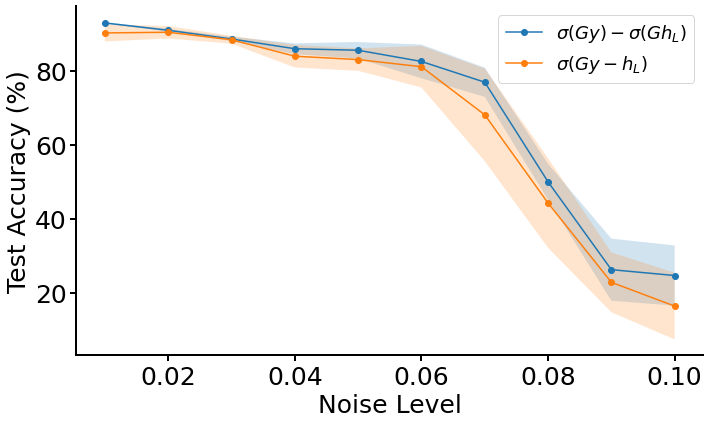}
    \caption{}
    \label{fig:projection_choice_4bit}
\end{subfigure}
\caption{Test accuracy on MNIST under programming noise of level $\alpha$ for the two target
projections, $\sigma(Gy)-\sigma(Gh_L)$ and $\sigma(G(y-h_L))$, at (a) 8-bit and (b) 4-bit
precision. Shaded regions denote the standard deviation over $10$ seeds.}
\label{fig:projection_choice}
\end{figure*}
\section{Memory and Runtime Efficiency of FTP}
\label{appendix_memory}

This section presents an analytical comparison, and the pipelining benefit described below is a property of the update schedule of FTP. We analyze the memory usage and runtime cost of FTP in comparison to BP and DFA \citep{nokland2016direct}, particularly in the context of resource-constrained hardware.

\paragraph{Memory Usage.}
The memory footprint during training arises from two primary sources: (i) activation storage during the forward pass, and (ii) storage of trainable parameters and feedback matrices.

\textit{Activation storage:} All methods (BP, DFA, FTP) require retaining activations during the forward pass until error signals or targets are computed. Thus, peak memory usage for storing activations is comparable across all three methods.

\textit{Parameter storage:} BP stores only the feedforward weights. FTP stores the same feedforward weights along with a single feedback matrix $\mG$. DFA, on the other hand, maintains a separate fixed feedback matrix for each hidden layer, resulting in the largest weight memory footprint. Notably, in deeper networks, the overhead of storing a single $\mG$ in FTP becomes negligible compared to the total number of trainable parameters.

\paragraph{Runtime Efficiency.}

FTP performs a standard forward pass, computes the final target, and then recursively generates intermediate layer targets while updating weights layer-by-layer. Unlike BP, which updates layers in reverse order and enforces update locking, FTP permits pipelined execution. Specifically, once the first layer’s weights are updated and the next target is being computed, the subsequent minibatch can begin its forward pass. This enables partial overlap between batches, potentially reducing wall-clock training time--particularly in hardware implementations that support pipelining.

DFA computes all layer gradients in parallel using fixed feedback matrices after the forward pass. The inter-minibatch delay in DFA corresponds to the time required to update all weights simultaneously. In contrast, for FTP, the delay is primarily determined by the computation of the first target and update of the first layer. Therefore, FTP's runtime should be comparable to DFA, with a modest overhead for target computation, and would improve on BP under a pipelined schedule. Verifying this on hardware that supports pipelined execution remains future work.


\section{Impact of Strict Weight Symmetry Requirement of BP in Emerging Analog Hardware}
\label{appendix_weightsymmtery_bp}

Backpropagation (BP) relies on backward matrices to propagate errors, which must maintain symmetry with the forward weight matrices in principle mathematically, which is easily achievable in the case of digital computers. However, in analog computing hardware such as RRAM and other non-volatile memories, non-idealities such as programming errors, thermal noise, and random telegraph noise can disrupt this symmetry, leading to significant degradation in performance. Our results, as shown in Figure~\ref{performance_comparison}, demonstrate the vulnerability of BP under these conditions, particularly in devices with low-bit precision. Here we considered the cases when some weight parameters in a matrix got corrupted due to device non-idealities, and the asymmetry here refers to the amount of weights getting corrupted by 10\% margin from where the values should be. For instance, in 4-bit systems \citep{4bitmem_WLu, 4bitmem_JY}, test accuracy drops below 80\% when only 20\% of the matrices are affected by such errors. The vulnerability is even more pronounced in 3-bit systems, where a mere 5\% mismatch between forward and backward matrices causes test accuracy to plummet below 60\%. As asymmetry increases, test accuracy steadily declines, eventually reaching approximately 15\%, which is equivalent to random guess. The standard deviation of test accuracy also increases, which indicates greater instability in performance. Each case was tested multiple times, and the shadowed region in Figure~\ref{performance_comparison} represents the standard deviation of test accuracy across trials. In contrast, FTP utilizes fixed, random backward matrices for target estimation, eliminating the dependency on symmetric forward-backward matrix relationships. This design inherently makes FTP resilient to hardware-induced noise and non-idealities, ensuring robust performance even in low-bit precision environments.
\begin{figure}[h]
\vspace{-3mm}
\centering
    \includegraphics[width=.5\linewidth]{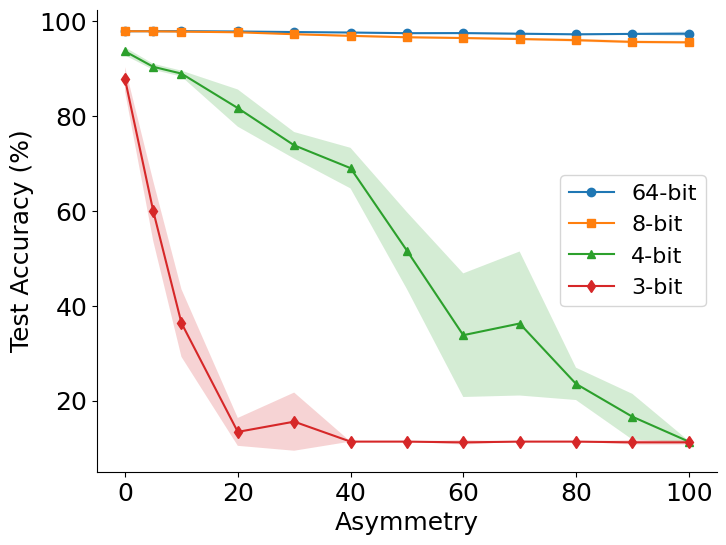}
    \vskip -.1in
    \caption{Performance of BP when asymmetry is introduced in backward matrices due to read noise.}
    \label{performance_comparison}
      \vskip -.1in
\end{figure}

\section{Compute Resources}
\label{appendix_compute}

All experiments were conducted on a local workstation configured with the following specifications: an AMD Ryzen Threadripper PRO 5955WX processor (16 cores, 4.00–4.50GHz), 256 GB of DDR4-3200 RAM, and dual NVIDIA GeForce RTX 4090 GPUs.

Experiments involving CNNs and RNNs were executed on the GPUs to leverage accelerated training, while fully connected network experiments and all evaluations in Section \ref{section_scalability}, Section \ref{section_alignment} and Section \ref{sec:robust_FTP} were performed on CPU only.

\section{Use of Large Language Models (LLMs)}
Large language models (LLMs) were used in this work solely as a general-purpose assistive tool for grammar checking and text editing. All research ideas, experimental designs, analyses, and scientific claims were conceived, implemented, and validated solely by the authors.

\end{document}